\documentclass[11pt]{article}

\usepackage[margin=1in]{geometry}
\usepackage{amsmath, amssymb, amsthm, mathtools}
\usepackage{mathrsfs}
\usepackage{bm}
\usepackage{bbm}
\usepackage{enumitem}
\usepackage{hyperref}
\usepackage{xcolor}
\usepackage{microtype}
\usepackage{graphicx}
\usepackage{float}
\usepackage{subcaption}
\usepackage{booktabs,tabularx,array,ragged2e}
\graphicspath{{figures/}{figures/exp1/}{figures/align/}}

\hypersetup{
  colorlinks=true,
  linkcolor=blue,
  citecolor=blue,
  urlcolor=blue
}

\theoremstyle{plain}
\newtheorem{theorem}{Theorem}[section]
\newtheorem{lemma}[theorem]{Lemma}
\newtheorem{proposition}[theorem]{Proposition}
\newtheorem{corollary}[theorem]{Corollary}

\theoremstyle{definition}
\newtheorem{definition}[theorem]{Definition}
\newtheorem{assumption}[theorem]{Assumption}

\theoremstyle{remark}
\newtheorem{remark}[theorem]{Remark}

\newcommand{\E}{\mathbb{E}}
\newcommand{\R}{\mathbb{R}}
\newcommand{\N}{\mathbb{N}}
\newcommand{\Id}{\mathrm{I}}
\newcommand{\diag}{\mathrm{diag}}
\newcommand{\tr}{\mathrm{tr}}

\newcommand{\eps}{\varepsilon}
\newcommand{\GL}{\mathrm{GL}}
\newcommand{\Orth}{\mathrm{O}}
\newcommand{\SPD}{\mathrm{SPD}}
\newcommand{\BW}{\mathrm{BW}}
\newcommand{\TV}{\mathrm{TV}}
\newcommand{\KL}{\mathrm{KL}}
\newcommand{\Harm}{\mathrm{H}}
\newcommand{\Diag}{\mathrm{Diag}}
\newcommand{\Diagp}{\mathrm{Diag}^{+}}
\newcommand{\Cart}{\operatorname{Cart}}

\DeclareMathOperator*{\argmin}{arg\,min}

\newcolumntype{Y}{>{\RaggedRight\arraybackslash\hspace{0pt}}X}
\newcolumntype{L}[1]{>{\RaggedRight\arraybackslash\hspace{0pt}}p{#1}}

\title{\vspace{-1.0em}
Geometric and Spectral Alignment for Deep Neural Network I\\
\vspace{-0.5em}}

\author{
Ziran Liu$^{1,5*}$,
Wei Wang$^{2*}$,
Jinhao Wang$^{3}$,
Pengcheng Wang$^{4}$,
Xinyi Sui$^{3}$\\
Cihan Ruan$^{3}$,
Nam Ling$^{3}$,
Wei Jiang$^{2}$\\[0.5em]
$^1$Shanghai Institute for Mathematics and Interdisciplinary Sciences (SIMIS), \\Shanghai 200433, China\\
$^2$Futurewei Technologies, Inc., San Jose, CA 95131\\
$^3$Dept. of Computer Science and Engineering, Santa Clara University,\\
Santa Clara, CA 95050\\
$^4$Dept. of Computer Science, Purdue University, West Lafayette, IN 47906\\
$^5$Research Institute of Intelligent Complex Systems, Fudan University,\\ Shanghai 200433, China\\
[0.4em]
\texttt{zliu@simis.cn, rickweiwang@futurewei.com, jwang11@scu.edu}\\
\texttt{wang4495@purdue.edu, xsui@scu.edu, luciacihanruan@gmail.com}\\
\texttt{nling@scu.edu, wjiang@futurewei.com}
}

\date{}

\usepackage[backend=biber,style=numeric,sorting=nyt]{biblatex}
\begin{document}
\maketitle

\begin{abstract}
Deep residual architectures can be modeled by finite products of near-identity Jacobians.  This paper proves deterministic, conditional quotient-geometric estimates for the singular spectra of Frobenius-normalized layer factors along such chains and, more precisely, for a normalized top-radial Cartan coordinate together with a specified fitted power-law chart.  Full-rank factors are projected from $\GL(d)$ to the positive definite cone by $A\mapsto A^\top A$ and then to ordered eigenvalue data.  Under Frobenius normalization, exact power-law spectra form a trace-normalized one-parameter Cartan orbit with generalized harmonic normalization.  We identify this orbit as a Gibbs family on ranks, a Fisher information line, and a Bures--Wasserstein curve whose line element is $d/4$ times the Fisher information.  The main rigidity theorem is a slack-aware margin inequality: a geometric-mean interface radial amplitude, measurable non-backtracking slack, and signed orbit-chart residual variation control the displacement of the fitted Cartan coordinate.  In the zero-slack exact-chart case, a depth-$L$ uniform budget gives adjacent exponent drift of order $(\log M)/L$; in the general case the same estimate is augmented by explicitly measurable slack and chart-residual increments.  We separate the scalar top-radial theorem from full-Cartan spectral control, which requires an additional Bures/Hellinger chart-residual variation term.  Approximate-power-law versions, metric-chart residual versions, converse lower bounds, Fisher--KL/Bures action estimates, and raw plus Frobenius-normalized residual near-identity expansions are proved.  The near-identity results verify the transport-budget side of the theorem, while chart quality remains a separate measurable hypothesis.  Finally, effective rank is formulated as a quantile of the spectral energy measure, yielding finite-width power-law tail bounds and robust actual rank-window transition estimates.  Empirical static-weight exponent profiles are presented as coordinate diagnostics; a full verification of the theorem additionally requires interface budgets, non-backtracking slacks, and chart residuals for the same operator chain.

\end{abstract}

\newpage
\tableofcontents
\newpage

\section{Introduction}
Residual networks and Transformer blocks use additive layer maps in which each state is updated by a learned perturbation of the current state \cite{he2016deep,vaswani2017attention}.  In deep residual models, stability is closely tied to how these perturbations are scaled with depth; this point is already visible in depth-scaled initialization schemes such as Fixup \cite{zhang2019fixup} and in large-depth analyses of residual networks \cite{marion2025scalingresnets}.  Along a trajectory,
\[
 x_{k+1}=x_k+F_k(x_k),
 \qquad
 J_k(x_k)=\Id+DF_k(x_k),
\]
and the linearized transport from the input to depth $L$ is the ordered product
\[
J(x_0)=J_{L-1}(x_{L-1})\cdots J_0(x_0).
\]
When the factors are invertible, this family of products lies in the real general linear group $\GL(d)$ and satisfies the finite cocycle identity $J_{n:m}J_{m:\ell}=J_{n:\ell}$.  In deep-learning terms, this cocycle is the object that transports infinitesimal activation perturbations and backpropagated signals across depth.  The spectral problem studied in this article asks how far the singular spectrum of the layer factors can move when this transport remains budgeted.

The answer is naturally formulated after removing coordinate gauge.  Left multiplication by an orthogonal matrix changes the output basis of a layer but does not change its singular values.  Therefore the spectral state of a full-rank matrix $A$ is represented by the positive definite form
\[
\pi(A)=A^\top A\in\SPD(d),
\]
which realizes the quotient $\Orth(d)\backslash\GL(d)$ in concrete matrix form.  The ordered eigenvalues of $\pi(A)$, or equivalently the squared singular values of $A$, lie in the positive Weyl chamber after Cartan projection.  The quotient-radial part of Geometric and Spectral Alignment (GSA) is the study of this Cartan-projected path.

Under Frobenius normalization, a finite-width power-law spectrum defines a trace-normalized one-parameter Cartan orbit
\[
G_d(\alpha)=\frac{d\,e^{-2\alpha D_d}}{\tr(e^{-2\alpha D_d})},
\qquad
D_d=\diag(\log 1,\ldots,\log d).
\]
This orbit has three equivalent interpretations used throughout the paper: it is a curve in the maximal flat of $\SPD(d)$, a Gibbs family on the rank set $\{1,\ldots,d\}$, and a Fisher information line whose Bures--Wasserstein length can be computed exactly.  Heavy-tailed/self-regularized spectral measurements of trained networks provide the empirical motivation for using such finite power-law coordinates \cite{martin2021predicting}; the results below are deterministic statements conditional on the stated orbit-chart and transport-budget hypotheses.

The main deterministic implication proved here is the more precise statement
\[
\begin{gathered}
\text{interface radial budget}
+\text{non-backtracking slack}
+\text{orbit-chart residual control}\\
\Longrightarrow
\text{short fitted Cartan-coordinate path}.
\end{gathered}
\]
The controlled radial observable is the normalized top-eigenvalue functional
$\rho_d(P)=\frac12\log(d\lambda_1(P)/\tr P)$.  Thus the first conclusion is a scalar quotient-radial estimate.  Full Cartan-spectral shortness is obtained only after adding an explicit full-spectrum chart error measuring how far $\Cart(P_k)$ lies from the fitted orbit point $G_d(\alpha_k)$.  This separation is important: the fitted exponent path can be short even when a layer has a poor power-law spectral chart, and the theorem records that defect as a measurable error term.  The empirical section therefore measures the visible coordinate $\widehat\alpha_k$ and states the corresponding margin quantities, namely local interface budgets, non-backtracking slacks, and chart residuals.

\subsection{Minimal geometric notation used in the main statements}
\label{sec:minimal-geometric-notation}

The main statements use only finite-dimensional matrix geometry.  The real general linear group is
\[
\GL(d):=\{A\in\R^{d\times d}:\det A\ne 0\},
\]
with matrix multiplication.  A finite multiplicative cocycle over the layer index set, in the deterministic finite-depth sense of cocycle transport \cite{arnold1998random}, is the family
\[
J_{n:m}:=J_{n-1}J_{n-2}\cdots J_m,\qquad 0\le m<n\le L,
\]
which satisfies $J_{n:m}J_{m:\ell}=J_{n:\ell}$ whenever $0\le\ell<m<n\le L$.  For a residual network, $J_{n:m}$ is the Jacobian transporting an infinitesimal perturbation from layer $m$ to layer $n$.

Let
\[
\Orth(d):=\{Q\in\R^{d\times d}:Q^\top Q=\Id\}
\]
be the orthogonal group and let
\[
\SPD(d):=\{P=P^\top:x^\top Px>0\text{ for all }x\ne0\}
\]
be the manifold of positive definite forms.  For $A\in\GL(d)$, the quotient projection used in the paper is
\[
\pi(A):=A^\top A\in\SPD(d).
\]
It removes left orthogonal gauge, since $(QA)^\top(QA)=A^\top A$ for every $Q\in\Orth(d)$.  Conversely, if $A^\top A=B^\top B$ for $A,B\in\GL(d)$, then $B=QA$ for $Q=BA^{-1}\in\Orth(d)$.  Hence $\SPD(d)$ realizes the quotient $\Orth(d)\backslash\GL(d)$ in concrete matrix form.

For $P\in\SPD(d)$, we use two equivalent radial representatives.  The positive Cartan representative is
\[
\Cart(P):=\diag(\lambda_1(P),\ldots,\lambda_d(P)),\qquad \lambda_1(P)\ge\cdots\ge\lambda_d(P)>0,
\]
and the logarithmic Cartan vector is
\[
\mathbf c(P):=(\log\lambda_1(P),\ldots,\log\lambda_d(P)).
\]
The positive representative is convenient for energy calculations, while $\mathbf c(P)$ is the usual Weyl-chamber coordinate in logarithmic form.  Both keep the ordered spectrum of the positive form and discard singular-vector gauge.  In deep-learning terms, this radial data controls amplification, normalized spectral shape, and effective-rank windows.  The angular singular-vector data is studied in the companion article on physical alignment.

\begin{table}[H]
\centering
\footnotesize
\renewcommand{\arraystretch}{1.24}
\setlength{\tabcolsep}{3.2pt}
\begin{tabularx}{\textwidth}{L{0.22\textwidth}L{0.36\textwidth}Y}
\toprule
\textbf{Object in this article} & \textbf{Formal definition} & \textbf{Deep-learning meaning}\\
\midrule
Residual Jacobian cocycle & Layerwise products $J_{n:m}=J_{n-1}\cdots J_m$; see Definition~\ref{def:jacobian-chain}. & Local linear transport of activations, perturbations, and backpropagated signals between depths.\\
Quotient projection & $\pi(A)=A^\top A\in\SPD(d)$; see Definition~\ref{def:quotient-cartan}. & Removes orthogonal coordinate gauge and keeps the amplification geometry of a layer.\\
Cartan coordinate & $g_d(\alpha)=\log\sqrt{d/\Harm_{d,2\alpha}}$ on the power-law orbit; see equation~\eqref{eq:g-def} and Proposition~\ref{prop:cartan-membership}. & Scalar coordinate of spectral shape; empirically estimated by fitting power-law singular spectra.\\
Spectral energy measure & $\mu_W=\sum_i \sigma_i(W)^2\|W\|_F^{-2}\delta_i$; see Definition~\ref{def:effective-rank}. & Energy distribution over singular directions, used to choose dominant rank windows.\\
Effective-rank window & $R_\eps(W)=\min\{r:\mu_W(\{1,\dots,r\})\ge1-\eps\}$; see Definition~\ref{def:effective-rank}. & Number of dominant singular channels needed to retain prescribed layer energy.\\
\bottomrule
\end{tabularx}
\caption{Finite-dimensional objects used in the spectral part of GSA.  The formal definitions are all stated in this article; the third column records the corresponding quantity in a deep-network layer.}
\label{tab:partI-object-dictionary}
\end{table}

\section{Geometric setup: residual Jacobian chains, quotient projection, and Cartan reduction}
\label{sec:model}

The basic object generated by a deep residual architecture is a multiplicative cocycle in $\GL(d)$.
Its spectral content is not a datum of $\GL(d)$ itself but of the quotient by orthogonal gauge.
This is the standard symmetric-space viewpoint on positive definite forms and Cartan/Weyl-chamber variables \cite{helgason2001differential,bhatia2007positive}.
We fix that quotient picture from the outset because the main theorems act separately on quotient-radial and angular data.

\subsection{Residual Jacobian cocycle in \texorpdfstring{$\GL(d)$}{GL(d)}}

Fix a width $d\in\N$ and a depth $L\in\N$.
Consider the residual recursion
\begin{equation}\label{eq:residual-recursion}
x_{k+1}=x_k+F_k(x_k),\qquad k=0,1,\dots,L-1,
\end{equation}
where each $F_k:\R^d\to\R^d$ is $C^1$.

\begin{definition}[Residual Jacobian chain as a discrete cocycle]
\label{def:jacobian-chain}
For a trajectory $(x_k)_{k=0}^{L}$ generated by \eqref{eq:residual-recursion}, define the layer Jacobians
\begin{equation}\label{eq:layer-jacobian}
J_k(x_k):=\frac{\partial x_{k+1}}{\partial x_k}=\Id+DF_k(x_k)\in\R^{d\times d},
\end{equation}
and the end-to-end Jacobian
\begin{equation}\label{eq:jacobian-chain}
J(x_0):=Df(x_0)=\prod_{k=0}^{L-1}J_k(x_k),
\end{equation}
with the convention that the rightmost factor is $J_0(x_0)$.
Whenever the factors are full rank, the sequence $(J_k)$ is a discrete path in the Lie group $\GL(d)$.
\end{definition}

\begin{lemma}[Chain rule in cocycle form]
\label{lem:chain-rule}
Under the setup above, the Jacobian of the network map is given by \eqref{eq:jacobian-chain}.
\end{lemma}

\begin{proof}
For each $k$, set $T_k(x):=x+F_k(x)$.  The residual recursion gives
\[
x_1=T_0(x_0),\quad x_2=T_1(T_0(x_0)),\quad\ldots,\quad
x_L=(T_{L-1}\circ\cdots\circ T_0)(x_0).
\]
Hence the network map is $f=T_{L-1}\circ\cdots\circ T_0$.  Since each $F_k$ is $C^1$, each $T_k$ is $C^1$, and the finite-dimensional chain rule gives
\[
Df(x_0)=DT_{L-1}(x_{L-1})DT_{L-2}(x_{L-2})\cdots DT_0(x_0).
\]
For every $k$, differentiating $T_k(x)=x+F_k(x)$ yields
\[
DT_k(x_k)=D(\Id)(x_k)+DF_k(x_k)=\Id+DF_k(x_k)=J_k(x_k).
\]
Substituting these factors into the previous display gives exactly the ordered product in \eqref{eq:jacobian-chain}.
\end{proof}

\subsection{Quotient projection to \texorpdfstring{$\SPD(d)$}{SPD(d)} and Cartan projection to the Weyl chamber}

For a full-rank matrix $A\in\GL(d)$, the positive form $A^\top A$ records exactly its singular-value content.
Geometrically, the map $A\mapsto A^\top A$ realizes the quotient $\Orth(d)\backslash\GL(d)$, while $A\mapsto AA^\top$ realizes $\GL(d)/\Orth(d)$.
Both are canonical realizations of the symmetric space $\SPD(d)$.
We work with the $A^\top A$ realization because it diagonalizes in the right singular basis.

\begin{definition}[Quotient projection and Cartan representatives]
\label{def:quotient-cartan}
Define the quotient projection
\[
\pi:\GL(d)\to\SPD(d),\qquad \pi(A):=A^\top A.
\]
Let $\Diag(d)$ denote the vector space of real diagonal $d\times d$ matrices and let $\Diagp(d)$ denote the cone of diagonal matrices with strictly positive diagonal entries.  For $P\in\SPD(d)$, define its positive Cartan representative by
\[
\Cart(P):=\diag(\lambda_1(P),\dots,\lambda_d(P)),
\]
where $\lambda_1(P)\ge\cdots\ge\lambda_d(P)>0$ are the ordered eigenvalues of $P$.
Thus $\Cart(P)$ lies in the positive Weyl chamber
\[
\mathfrak A_d^{+}:=\{\diag(\lambda_1,\dots,\lambda_d):\lambda_1\ge\cdots\ge\lambda_d>0\}\subset \Diagp(d)\subset\SPD(d).
\]
The associated logarithmic Cartan vector is
\[
\mathbf c(P):=(\log\lambda_1(P),\dots,\log\lambda_d(P))\in\{u_1\ge\cdots\ge u_d\}\subset\R^d.
\]
\end{definition}

\begin{remark}[Convention for the Cartan coordinate]
In the Lie-theoretic convention for $\GL(d)/\Orth(d)$, the Cartan projection is usually the logarithmic singular-value vector.  Since $P=A^\top A$ has eigenvalues $\sigma_i(A)^2$, the vector $\frac12\mathbf c(P)$ is exactly the usual ordered log-singular-value vector of $A$.  The paper keeps both $\Cart(P)$ and $\mathbf c(P)$ explicit: $\Cart(P)$ is used for energy and trace-normalized formulas, while $\mathbf c(P)$ records the logarithmic Weyl-chamber coordinate.
\end{remark}

\begin{remark}[Radial and angular variables]
The projected path $P_k=\pi(J_k)$ in $\SPD(d)$ contains an eigenvalue component and an angular component.
The Cartan projection $\Cart(P_k)$ keeps the ordered eigenvalue data and removes singular-vector gauge.
The spectral rigidity theorem controls this Cartan-projected path.
The remaining singular-vector transport is treated through physical alignment structures in the companion article.
This separation preserves the quotient-radial geometry and the angular transport geometry as distinct mathematical objects.
\end{remark}

\begin{definition}[Trace-normalized slice]
\label{def:trace-slice}
Let
\[
\SPD_d^{(1)}:=\{P\in\SPD(d):\tr P=d\}.
\]
Under layerwise Frobenius normalization, the projected matrices $P_k=W_k^\top W_k$ (in particular $J_k^\top J_k$ when $W_k=J_k$) lie in this trace-$d$ slice.
\end{definition}

\subsection{Square reduction as a spectral embedding}

Modern architectures often change width across layers.
For the purposes of spectral geometry, rectangular blocks may be embedded into a common square ambient space without changing any nonzero spectral data.

\begin{lemma}[Square embedding preserves nonzero singular values]
\label{lem:square-embedding}
Let $W\in\R^{m\times n}$ and set $d:=\max\{m,n\}$.
Define the padded square matrix $\widetilde W\in\R^{d\times d}$ by
\begin{equation}\label{eq:square-padding-definition}
(\widetilde W)_{ab}:=
\begin{cases}
W_{ab}, & 1\le a\le m,\ 1\le b\le n,\\
0, & \text{otherwise}.
\end{cases}
\end{equation}
Then:
\begin{enumerate}[leftmargin=1.6em]
\item the multiset of singular values of $\widetilde W$, counted as a square $d\times d$ matrix, is the multiset of singular values of $W$ together with $|m-n|$ additional zeros;
\item in particular,
\[
\|\widetilde W\|_2=\|W\|_2,
\qquad
\|\widetilde W\|_F=\|W\|_F.
\]
\end{enumerate}
\end{lemma}

\begin{proof}
We distinguish the two possible rectangular shapes, because the number of singular values conventionally attached to a rectangular $m\times n$ matrix is $\min\{m,n\}$, whereas the padded square matrix has $d=\max\{m,n\}$ singular values.

First suppose $m\le n$.  Then $d=n$, and \eqref{eq:square-padding-definition} means that $\widetilde W$ is obtained from $W$ by adding $n-m$ zero rows.  Hence
\[
\widetilde W^{\top}\widetilde W=W^{\top}W.
\]
The matrix $W^{\top}W$ is $n\times n$ and has eigenvalues consisting of the $m$ squared singular values of $W$ together with $n-m$ zeros.  Since the squared singular values of $\widetilde W$ are the eigenvalues of the same matrix $\widetilde W^{\top}\widetilde W$, the singular values of $\widetilde W$ are those of $W$ together with $n-m=|m-n|$ additional zeros.

Now suppose $m>n$.  Then $d=m$, and $\widetilde W$ is obtained from $W$ by adding $m-n$ zero columns.  In this case
\[
\widetilde W^{\top}\widetilde W=
\begin{bmatrix}
W^{\top}W&0\\
0&0
\end{bmatrix},
\]
where the lower-right zero block has size $(m-n)\times(m-n)$.  The eigenvalues of this block diagonal matrix are the eigenvalues of $W^{\top}W$ together with $m-n$ additional zeros.  Taking nonnegative square roots again gives the singular values of $W$ together with $|m-n|$ additional zeros.

The operator norm is the largest singular value, and the Frobenius norm is the square root of the sum of squared singular values.  Adding only zero singular values changes neither quantity, proving the two norm equalities.
\end{proof}

\begin{remark}[Geometric meaning of padding]
Zero-padding is a spectrally isometric embedding for the nonzero singular values and for all unitarily invariant quantities used in the finite-rank measurements.  It does not, however, turn a rectangular or rank-deficient operator into an interior point of $\SPD(d)$: the Gram matrix $\widetilde W^\top\widetilde W$ lies on the positive-semidefinite boundary whenever extra zero singular values are present.  Consequently, the symmetric-space theorems below are stated for full-rank square Jacobian factors, while rectangular static sublayers are handled by their nonzero spectral stratum or, equivalently, by the regularized forms
\[
\widetilde W^\top\widetilde W+\varepsilon\Id\in\SPD(d),
\qquad \varepsilon>0,
\]
with all reported exponent, effective-rank, and alignment measurements computed from the nonzero singular spectrum before the limit $\varepsilon\downarrow0$.  The Bures--Wasserstein formula extends continuously to the positive-semidefinite cone, so Bures/Hellinger chart-error diagnostics can also be evaluated on the PSD closure; the main rigidity theorem below remains stated on the full-rank trace-normalized slice.
\end{remark}


\section{The Cartan power-law orbit: Gibbs family, Fisher geometry, and Bures geometry}
\label{sec:spectral-orbit}

The Cartan spectral component is formulated on a one-dimensional submanifold of the trace-normalized positive cone.
It is simultaneously:
\begin{enumerate}[leftmargin=2.0em]
\item a normalized Cartan ray in the maximal flat of $\SPD(d)$;
\item a Gibbs family on the finite rank set $\{1,\dots,d\}$;
\item a one-dimensional statistical manifold whose Fisher metric and the pulled-back Bures--Wasserstein metric are related by the exact factor in Theorem~\ref{thm:fisher-bures}.
\end{enumerate}
The construction provides the metric structure used below to convert radial budget bounds into Fisher--Bures path-length bounds.  All nontrivial rigidity statements below assume $d\ge2$; for $d=1$ the orbit is a single point, $g_d'\equiv0$, and the exponent coordinate is not identifiable from spectral shape.

\subsection{Cartan generator, harmonic normalization, and exact orbit membership}

\begin{definition}[Generalized harmonic numbers]
\label{def:harmonic}
For $d\in\N$ and $s\in\R$, define
\[
\Harm_{d,s}:=\sum_{i=1}^d i^{-s}.
\]
When $s>1$, one has $\Harm_{d,s}\uparrow \zeta(s)$ as $d\to\infty$.
\end{definition}

\begin{definition}[Cartan generator and normalized power-law orbit]
\label{def:cartan-orbit}
Let
\[
D_d:=\diag(\log 1,\log 2,\dots,\log d)\in\Diag(d),
\qquad
A_d(\alpha):=e^{-\alpha D_d}=\diag(1,2^{-\alpha},\dots,d^{-\alpha}).
\]
Define the normalized Cartan power-law orbit by
\begin{equation}\label{eq:canonical-gram}
G_d(\alpha):=\frac{d}{\tr(A_d(\alpha)^2)}A_d(\alpha)^2
=\diag\!\left(\lambda_1(\alpha),\dots,\lambda_d(\alpha)\right),
\qquad
\lambda_i(\alpha)=\frac{d}{\Harm_{d,2\alpha}}i^{-2\alpha}.
\end{equation}
Then $G_d(\alpha)\in\mathfrak A_d^{+}\cap\SPD_d^{(1)}$ for every $\alpha>0$.
\end{definition}

\begin{assumption}[Exact power-law singular values]
\label{ass:power-law}
Fix a width $d$.
For each layer factor $W_k\in\R^{d\times d}$ of interest, its singular values satisfy
\begin{equation}\label{eq:power-law}
\sigma_i(W_k)=C_k\,i^{-\alpha_k},\qquad i=1,2,\dots,d,
\end{equation}
for some parameters $C_k>0$ and $\alpha_k>0$.
\end{assumption}

\begin{assumption}[Layerwise Frobenius normalization]
\label{ass:frob-norm}
Throughout the main theorem package, we impose
\begin{equation}\label{eq:frob-normalization}
\|W_k\|_F^2=d,\qquad \text{for all }k.
\end{equation}
\end{assumption}

\begin{proposition}[Exact orbit membership in the Cartan chamber]
\label{prop:cartan-membership}
Let $P_k:=W_k^\top W_k$.
Under Assumption~\ref{ass:power-law} and Assumption~\ref{ass:frob-norm}, one has
\begin{equation}\label{eq:cartan-membership}
\Cart(P_k)=G_d(\alpha_k).
\end{equation}
Equivalently, the Cartan projection of the layer factor lies exactly on the normalized power-law orbit.
Moreover,
\begin{align}
\|W_k\|_2 &= \sqrt{\frac{d}{\Harm_{d,2\alpha_k}}}, \label{eq:top-sval-alpha}\\
\log\|W_k\|_2 &= g_d(\alpha_k), \qquad g_d(\alpha):=\frac12\log d-\frac12\log \Harm_{d,2\alpha}. \label{eq:g-def}
\end{align}
Finally,
\begin{equation}\label{eq:g-derivative}
g_d'(\alpha)=\frac{\sum_{i=1}^d (\log i)\,i^{-2\alpha}}{\Harm_{d,2\alpha}}.
\end{equation}
If $d\ge2$, then $g_d'(\alpha)\in(0,\log d]$; for $d=1$ the coordinate is degenerate and $g_d'\equiv0$.
\end{proposition}

\begin{proof}
The eigenvalues of $P_k=W_k^\top W_k$ are the squared singular values of $W_k$.
Under the exact power-law ansatz,
\[
\lambda_i(P_k)=\sigma_i(W_k)^2=C_k^2 i^{-2\alpha_k}.
\]
Frobenius normalization gives
\[
d=\|W_k\|_F^2=C_k^2\Harm_{d,2\alpha_k},
\]
so $C_k^2=d/\Harm_{d,2\alpha_k}$ and therefore
\[
\lambda_i(P_k)=\frac{d}{\Harm_{d,2\alpha_k}}i^{-2\alpha_k}=\lambda_i(\alpha_k).
\]
This proves \eqref{eq:cartan-membership}.
Equation \eqref{eq:top-sval-alpha} follows from $\|W_k\|_2=\sigma_1(W_k)=C_k$, and \eqref{eq:g-def} is just its logarithm.
Differentiating \eqref{eq:g-def} yields \eqref{eq:g-derivative}.
\end{proof}

\begin{definition}[Radial spectral functional]
\label{def:radial-functional}
For $P\in\SPD(d)$ define
\begin{equation}\label{eq:rho-def}
\rho_d(P):=\frac12\log\frac{d\,\lambda_1(P)}{\tr P}.
\end{equation}
When $P\in\SPD_d^{(1)}$, this simplifies to
\[
\rho_d(P)=\frac12\log\lambda_1(P).
\]
Along the Cartan power-law orbit,
\begin{equation}\label{eq:rho-on-orbit}
\rho_d(G_d(\alpha))=g_d(\alpha).
\end{equation}
\end{definition}

\begin{remark}[Radial-coordinate interpretation]
The function $\rho_d$ is the radial functional that the interface budget directly controls.
The function $g_d$ is the same functional restricted to the normalized power-law orbit.
Thus the rigidity theorem is a genuine coordinate-projection theorem from a radial functional on the symmetric space to a coordinate on a Cartan orbit.
\end{remark}

\subsection{Gibbs family and information geometry on the rank set}

\begin{definition}[Canonical spectral energy distribution]
\label{def:spectral-energy-distribution}
Under Assumption~\ref{ass:power-law} and Assumption~\ref{ass:frob-norm}, define
\begin{equation}\label{eq:pi-alpha}
p_i^{(\alpha)}:=\frac{\lambda_i(\alpha)}{d}=\frac{i^{-2\alpha}}{\Harm_{d,2\alpha}},
\qquad i=1,\dots,d.
\end{equation}
Then $p^{(\alpha)}=(p_1^{(\alpha)},\dots,p_d^{(\alpha)})$ is a probability vector on $\{1,\dots,d\}$.
\end{definition}

\begin{remark}[Gibbs realization]
Equation \eqref{eq:pi-alpha} is exactly a Gibbs law on the finite energy set $E_i=\log i$ with inverse temperature $2\alpha$:
\[
p_i^{(\alpha)}=\frac{e^{-2\alpha E_i}}{\sum_{j=1}^d e^{-2\alpha E_j}}.
\]
Thus the power-law orbit is at the same time a Cartan orbit in $\SPD(d)$ and a statistical manifold of Gibbs states~\cite{amari2016information}.
\end{remark}

\begin{definition}[Internal energy, variance, and Fisher information]
\label{def:internal-fisher}
Define
\[
U_d(\alpha):=\sum_{i=1}^d (\log i)\,p_i^{(\alpha)},
\qquad
V_d(\alpha):=\operatorname{Var}_{p^{(\alpha)}}(\log i),
\]
and
\begin{equation}\label{eq:Fisher-info}
I_d(\alpha):=\E_{p^{(\alpha)}}\!\left[\left(\frac{d}{d\alpha}\log p_i^{(\alpha)}\right)^2\right].
\end{equation}
\end{definition}

\begin{proposition}[Thermodynamic identities on the Gibbs--Cartan family]
\label{prop:g-concavity}
Assume $d\ge2$. For every $\alpha>0$,
\begin{align}
U_d(\alpha) &= g_d'(\alpha), \label{eq:Ud-def}\\
g_d''(\alpha) &= -2V_d(\alpha)<0, \label{eq:g-second-derivative}\\
I_d(\alpha) &= 4V_d(\alpha)=\frac{d^2}{d\alpha^2}\log\Harm_{d,2\alpha}. \label{eq:Fisher-info-variance}
\end{align}
In particular, $g_d$ is strictly increasing and strictly concave.
\end{proposition}

\begin{proof}
We first identify the internal-energy coordinate.  By Definition~\ref{def:spectral-energy-distribution},
\[
p_i^{(\alpha)}=\frac{i^{-2\alpha}}{\Harm_{d,2\alpha}}.
\]
Therefore
\[
\sum_{i=1}^d(\log i)p_i^{(\alpha)}
=\frac{\sum_{i=1}^d(\log i)i^{-2\alpha}}{\Harm_{d,2\alpha}},
\]
which is exactly the derivative formula \eqref{eq:g-derivative}.  Hence $U_d(\alpha)=g_d'(\alpha)$.

We next compute the derivative of the Gibbs weights.  Since
\[
\log p_i^{(\alpha)}=-2\alpha\log i-\log\Harm_{d,2\alpha},
\]
Proposition~\ref{prop:cartan-membership} gives
\[
\frac{d}{d\alpha}\log\Harm_{d,2\alpha}=-2g_d'(\alpha)=-2U_d(\alpha).
\]
Thus
\[
\frac{d}{d\alpha}\log p_i^{(\alpha)}=-2\log i+2U_d(\alpha)=-2(\log i-U_d(\alpha)).
\]
Multiplying by $p_i^{(\alpha)}$ yields
\[
\frac{d}{d\alpha}p_i^{(\alpha)}=-2(\log i-U_d(\alpha))p_i^{(\alpha)}.
\]
Now differentiate $g_d'(\alpha)=U_d(\alpha)=\sum_i(\log i)p_i^{(\alpha)}$:
\begin{align*}
g_d''(\alpha)
&=\sum_{i=1}^d(\log i)\frac{d}{d\alpha}p_i^{(\alpha)}\\
&=-2\sum_{i=1}^d(\log i)(\log i-U_d(\alpha))p_i^{(\alpha)}\\
&=-2\left(\sum_{i=1}^d(\log i)^2p_i^{(\alpha)}-U_d(\alpha)\sum_{i=1}^d(\log i)p_i^{(\alpha)}\right)\\
&=-2\left(\E_{p^{(\alpha)}}[(\log i)^2]-U_d(\alpha)^2\right)
=-2V_d(\alpha).
\end{align*}
Finally,
\[
I_d(\alpha)=\sum_i p_i^{(\alpha)}\left(\frac{d}{d\alpha}\log p_i^{(\alpha)}\right)^2
=4\sum_i(\log i-U_d(\alpha))^2p_i^{(\alpha)}=4V_d(\alpha).
\]
For $d\ge2$, the random variable $\log i$ takes at least the two distinct values $0$ and $\log2$ with positive probability under $p^{(\alpha)}$, so $V_d(\alpha)>0$.
\end{proof}

\begin{lemma}[Conditioning of the top-radial exponent chart]
\label{lem:md-conditioning}
Let $d\ge2$ and let $I=[\alpha_{\min},\alpha_{\max}]\subset(0,\infty)$.  Then
\begin{equation}\label{eq:md-endpoint}
m_d(I):=\min_{\alpha\in I}g_d'(\alpha)=g_d'(\alpha_{\max}).
\end{equation}
Moreover,
\begin{equation}\label{eq:md-lower-bound}
m_d(I)
\ge
\frac{(\log 2)2^{-2\alpha_{\max}}}{\Harm_{d,2\alpha_{\min}}}.
\end{equation}
\end{lemma}

\begin{proof}
By Proposition~\ref{prop:g-concavity}, $g_d''(\alpha)=-2V_d(\alpha)<0$ for $d\ge2$, so $g_d'$ is strictly decreasing on $I$ and the minimum occurs at $\alpha_{\max}$.  For the lower bound, keep only the $i=2$ term in the numerator of \eqref{eq:g-derivative} and use $\Harm_{d,2\alpha_{\max}}\le\Harm_{d,2\alpha_{\min}}$ because $\alpha_{\max}\ge\alpha_{\min}$.
\end{proof}

\begin{remark}[Identifiability of the exponent coordinate]
The factor $1/m_d(I)$ in the rigidity theorem is an identifiability constant for recovering the exponent from the normalized top-radial observable.  When $\alpha_{\max}$ is large, the power-law spectrum is already close to rank-one, $g_d'(\alpha)$ becomes small, and top-radial measurements identify $\alpha$ poorly.  This conditioning is intrinsic to the one-dimensional chart.
\end{remark}

\begin{definition}[Spectral entropy and relative entropy]
\label{def:spectral-entropies}
Define the spectral Shannon entropy
\[
S_d(\alpha):=-\sum_{i=1}^d p_i^{(\alpha)}\log p_i^{(\alpha)},
\]
and, for $\alpha,\beta>0$, the spectral relative entropy
\[
\KL_d(\alpha\|\beta):=\sum_{i=1}^d p_i^{(\alpha)}\log\frac{p_i^{(\alpha)}}{p_i^{(\beta)}}.
\]
\end{definition}

\begin{proposition}[Closed entropy and KL formulas]
\label{prop:entropy-KL-Fisher}
For $\alpha,\beta>0$,
\begin{align}
S_d(\alpha)
&=
\log \Harm_{d,2\alpha}+2\alpha\,g_d'(\alpha), \label{eq:entropy-closed-form}\\
\KL_d(\alpha\|\beta)
&=
\log\frac{\Harm_{d,2\beta}}{\Harm_{d,2\alpha}}+2(\beta-\alpha)\,g_d'(\alpha). \label{eq:KL-closed-form}
\end{align}
Moreover, there exists $\xi$ between $\alpha$ and $\beta$ such that
\begin{equation}\label{eq:KL-quadratic}
\KL_d(\alpha\|\beta)=\frac12 I_d(\xi)(\beta-\alpha)^2.
\end{equation}
Finally, $I_d(\alpha)\le (\log d)^2$ for all $\alpha>0$.
\end{proposition}

\begin{proof}
For the entropy formula, use
\[
\log p_i^{(\alpha)}=-2\alpha\log i-\log \Harm_{d,2\alpha}.
\]
Then
\begin{align*}
S_d(\alpha)
&=-\sum_{i=1}^d p_i^{(\alpha)}\log p_i^{(\alpha)}\\
&=\sum_{i=1}^d p_i^{(\alpha)}\left(2\alpha\log i+\log\Harm_{d,2\alpha}\right)\\
&=2\alpha\sum_{i=1}^d(\log i)p_i^{(\alpha)}+\log\Harm_{d,2\alpha}\sum_{i=1}^dp_i^{(\alpha)}.
\end{align*}
Since the weights sum to one and $\sum_i(\log i)p_i^{(\alpha)}=g_d'(\alpha)$ by Proposition~\ref{prop:g-concavity}, this gives \eqref{eq:entropy-closed-form}.

For the KL formula, compute pointwise
\begin{align*}
\log\frac{p_i^{(\alpha)}}{p_i^{(\beta)}}
&=(-2\alpha\log i-\log\Harm_{d,2\alpha})-(-2\beta\log i-\log\Harm_{d,2\beta})\\
&=2(\beta-\alpha)\log i+\log\frac{\Harm_{d,2\beta}}{\Harm_{d,2\alpha}}.
\end{align*}
Averaging with respect to $p^{(\alpha)}$ yields
\[
\KL_d(\alpha\|\beta)=2(\beta-\alpha)g_d'(\alpha)+\log\frac{\Harm_{d,2\beta}}{\Harm_{d,2\alpha}},
\]
which is \eqref{eq:KL-closed-form}.

Set $A(t):=\log\Harm_{d,2t}$.  Proposition~\ref{prop:g-concavity} gives $A''(t)=I_d(t)$, and the KL formula can be rewritten as
\[
\KL_d(\alpha\|\beta)=A(\beta)-A(\alpha)-A'(\alpha)(\beta-\alpha).
\]
Taylor's theorem with Lagrange remainder gives a point $\xi$ between $\alpha$ and $\beta$ such that
\[
A(\beta)=A(\alpha)+A'(\alpha)(\beta-\alpha)+\frac12A''(\xi)(\beta-\alpha)^2.
\]
Substituting into the preceding display gives \eqref{eq:KL-quadratic}.  Finally, $\log i\in[0,\log d]$, so Popoviciu's variance bound gives $V_d(\alpha)\le(\log d)^2/4$; since $I_d=4V_d$, one obtains $I_d(\alpha)\le(\log d)^2$.
\end{proof}

\subsection{Main geometric theorem: the Cartan orbit is Fisher--Bures}

\begin{definition}[Bures--Wasserstein distance on $\SPD(d)$]
\label{def:BW}
For $P,Q\in\SPD(d)$ define
\begin{equation}\label{eq:BW-def}
d_{\BW}(P,Q)^2:=\tr(P)+\tr(Q)-2\,\tr\!\Bigl((P^{1/2}QP^{1/2})^{1/2}\Bigr).
\end{equation}
\end{definition}
The formula follows the standard Bures--Wasserstein geometry of positive definite matrices~\cite{bhatia2019bures,takatsu2011wasserstein}.  All Fisher--Bures constants in this paper use the distance convention in Definition~\ref{def:BW}; other common normalizations differ only by fixed scalar factors.

\begin{theorem}[Fisher--Bures identification on the Cartan power-law orbit]
\label{thm:fisher-bures}
Let $\iota_d:(0,\infty)\to\SPD_d^{(1)}$ be the curve $\iota_d(\alpha)=G_d(\alpha)$.
Then the following hold.
\begin{enumerate}[label=(F\arabic*),leftmargin=2.1em]
\item \textbf{Closed distance formula.}
For every $\alpha,\beta>0$,
\begin{equation}\label{eq:BW-harmonic}
d_{\BW}\bigl(G_d(\alpha),G_d(\beta)\bigr)^2
=
2d\left(
1-\frac{\Harm_{d,\alpha+\beta}}{\sqrt{\Harm_{d,2\alpha}\,\Harm_{d,2\beta}}}
\right).
\end{equation}
Equivalently,
\begin{equation}\label{eq:BW-hellinger}
d_{\BW}\bigl(G_d(\alpha),G_d(\beta)\bigr)^2
=
d\,\bigl\|\sqrt{p^{(\alpha)}}-\sqrt{p^{(\beta)}}\bigr\|_2^2,
\end{equation}
where the square root is taken entrywise.

\item \textbf{Induced line element.}
The Bures--Wasserstein metric pulled back to the Cartan orbit satisfies
\begin{equation}\label{eq:BW-line-element}
(\iota_d^{*}g_{\BW})_{\alpha}(\dot\alpha,\dot\alpha)
=
d\,V_d(\alpha)\,\dot\alpha^2
=
\frac d4 I_d(\alpha)\,\dot\alpha^2.
\end{equation}
Equivalently,
\begin{equation}\label{eq:BW-infinitesimal}
\lim_{h\to 0}\frac{d_{\BW}(G_d(\alpha+h),G_d(\alpha))^2}{h^2}
=
d\,V_d(\alpha).
\end{equation}

\item \textbf{Length coordinate and interval bounds.}
If $I=[a,b]\subset(0,\infty)$, then the Bures length of the restricted orbit is
\begin{equation}\label{eq:BW-length-formula}
L_{\BW}(\iota_d|_I)=\int_a^b \sqrt{d\,V_d(t)}\,dt.
\end{equation}
In particular, for any compact interval $I\subset(0,\infty)$ and any $\alpha,\beta\in I$,
\begin{equation}\label{eq:BW-lipschitz}
d_{\BW}(G_d(\alpha),G_d(\beta))
\le
\sqrt{d\,V_{d,\max}(I)}\,|\alpha-\beta|
\le
\frac{\sqrt d\,\log d}{2}|\alpha-\beta|,
\end{equation}
where
\[
V_{d,\max}(I):=\max_{t\in I}V_d(t).
\]
\end{enumerate}
\end{theorem}

\begin{proof}
We prove the three assertions in the order in which they are stated.
\begin{enumerate}[label=(F\arabic*),leftmargin=2.1em]
\item \textbf{Closed distance formula.}
Write
\[
G_d(\alpha)=\diag(\lambda_1(\alpha),\dots,\lambda_d(\alpha)),
\qquad
\lambda_i(\alpha)=\frac{d}{\Harm_{d,2\alpha}}i^{-2\alpha}.
\]
Both $G_d(\alpha)$ and $G_d(\beta)$ are diagonal with strictly positive diagonal entries.  Therefore they commute, and their positive square roots are obtained entrywise.  Hence
\[
G_d(\alpha)^{1/2}G_d(\beta)G_d(\alpha)^{1/2}
=
\diag\bigl(\lambda_i(\alpha)\lambda_i(\beta)\bigr)_{i=1}^d.
\]
Taking the positive square root of this diagonal positive definite matrix gives
\[
\Bigl(G_d(\alpha)^{1/2}G_d(\beta)G_d(\alpha)^{1/2}\Bigr)^{1/2}
=
\diag\bigl(\sqrt{\lambda_i(\alpha)\lambda_i(\beta)}\bigr)_{i=1}^d.
\]
Substituting this expression into the definition of $d_{\BW}$ in \eqref{eq:BW-def} gives
\begin{align*}
d_{\BW}\bigl(G_d(\alpha),G_d(\beta)\bigr)^2
&=\tr G_d(\alpha)+\tr G_d(\beta)
 -2\sum_{i=1}^d\sqrt{\lambda_i(\alpha)\lambda_i(\beta)}.
\end{align*}
By Definition~\ref{def:cartan-orbit},
\[
\tr G_d(\alpha)=\sum_{i=1}^d\frac{d}{\Harm_{d,2\alpha}}i^{-2\alpha}=d,
\qquad
\tr G_d(\beta)=d.
\]
For the mixed term, compute directly:
\begin{align*}
\sum_{i=1}^d\sqrt{\lambda_i(\alpha)\lambda_i(\beta)}
&=\sum_{i=1}^d
\sqrt{\frac{d}{\Harm_{d,2\alpha}}i^{-2\alpha}}
\sqrt{\frac{d}{\Harm_{d,2\beta}}i^{-2\beta}}\\
&=\frac{d}{\sqrt{\Harm_{d,2\alpha}\Harm_{d,2\beta}}}
\sum_{i=1}^d i^{-(\alpha+\beta)}\\
&=d\frac{\Harm_{d,\alpha+\beta}}{\sqrt{\Harm_{d,2\alpha}\Harm_{d,2\beta}}}.
\end{align*}
Combining the trace terms and the mixed term proves \eqref{eq:BW-harmonic}.
Since $p_i^{(\alpha)}=\lambda_i(\alpha)/d$, one has
\[
\sum_{i=1}^d\sqrt{\lambda_i(\alpha)\lambda_i(\beta)}
=d\sum_{i=1}^d\sqrt{p_i^{(\alpha)}p_i^{(\beta)}}.
\]
Thus
\begin{align*}
d_{\BW}\bigl(G_d(\alpha),G_d(\beta)\bigr)^2
&=2d\left(1-\sum_{i=1}^d\sqrt{p_i^{(\alpha)}p_i^{(\beta)}}\right)\\
&=d\sum_{i=1}^d\left(\sqrt{p_i^{(\alpha)}}-\sqrt{p_i^{(\beta)}}\right)^2,
\end{align*}
because $\sum_i p_i^{(\alpha)}=\sum_i p_i^{(\beta)}=1$.  This proves \eqref{eq:BW-hellinger}.

\item \textbf{Induced line element.}
Define
\[
q_i(\alpha):=\sqrt{p_i^{(\alpha)}}
=\exp\left(-\alpha\log i-\frac12\log \Harm_{d,2\alpha}\right).
\]
Since $g_d(\alpha)=\frac12\log d-\frac12\log\Harm_{d,2\alpha}$, differentiating gives
\[
\frac{d}{d\alpha}\log\Harm_{d,2\alpha}=-2g_d'(\alpha)=-2U_d(\alpha).
\]
Therefore
\begin{align*}
\frac{d}{d\alpha}\log q_i(\alpha)
&=-\log i-\frac12\frac{d}{d\alpha}\log \Harm_{d,2\alpha}\\
&=-\log i+U_d(\alpha)
=-(\log i-U_d(\alpha)).
\end{align*}
Multiplying by $q_i(\alpha)$ yields
\[
q_i'(\alpha)=-(\log i-U_d(\alpha))q_i(\alpha).
\]
Hence
\begin{align*}
\|q'(\alpha)\|_2^2
&=\sum_{i=1}^d (\log i-U_d(\alpha))^2q_i(\alpha)^2\\
&=\sum_{i=1}^d (\log i-U_d(\alpha))^2p_i^{(\alpha)}
=V_d(\alpha).
\end{align*}
Now apply \eqref{eq:BW-hellinger} with $\beta=\alpha+h$:
\[
\frac{d_{\BW}(G_d(\alpha+h),G_d(\alpha))^2}{h^2}
=d\left\|\frac{q(\alpha+h)-q(\alpha)}{h}\right\|_2^2.
\]
Since $q$ is smooth as a finite-dimensional map, the difference quotient converges to $q'(\alpha)$ as $h\to0$.  Passing to the limit gives
\[
\lim_{h\to0}\frac{d_{\BW}(G_d(\alpha+h),G_d(\alpha))^2}{h^2}
=d\|q'(\alpha)\|_2^2=dV_d(\alpha),
\]
which proves \eqref{eq:BW-infinitesimal}.  Proposition~\ref{prop:g-concavity} gives $I_d(\alpha)=4V_d(\alpha)$, so the pulled-back quadratic form is
\[
(\iota_d^*g_{\BW})_\alpha(\dot\alpha,\dot\alpha)=dV_d(\alpha)\dot\alpha^2=\frac d4 I_d(\alpha)\dot\alpha^2.
\]
This proves \eqref{eq:BW-line-element}.

\item \textbf{Length coordinate and interval bounds.}
For a $C^1$ curve $\alpha(t)$ in the parameter interval, the preceding line element gives the Bures speed
\[
\left\|\frac{d}{dt}G_d(\alpha(t))\right\|_{\BW}
=
\sqrt{dV_d(\alpha(t))}\,|\alpha'(t)|.
\]
Taking the monotone parametrization $\alpha(t)=t$ on $[a,b]$ therefore gives
\[
L_{\BW}(\iota_d|_I)=\int_a^b\sqrt{dV_d(t)}\,dt,
\]
which is \eqref{eq:BW-length-formula}.  For arbitrary $\alpha,\beta\in I$, assume first that $\alpha<\beta$; the case $\beta<\alpha$ is identical after exchanging the endpoints, and the case $\alpha=\beta$ is trivial.  From \eqref{eq:BW-hellinger}, with $q(t)=\sqrt{p^{(t)}}$, we have
\[
d_{\BW}(G_d(\alpha),G_d(\beta))
=\sqrt d\,\|q(\beta)-q(\alpha)\|_2.
\]
Since $q$ is continuously differentiable as a finite-dimensional vector-valued map,
\[
q(\beta)-q(\alpha)=\int_\alpha^\beta q'(t)\,dt.
\]
Taking norms and using the triangle inequality for the finite-dimensional integral gives
\[
\|q(\beta)-q(\alpha)\|_2
\le
\int_\alpha^\beta \|q'(t)\|_2\,dt
=
\int_\alpha^\beta \sqrt{V_d(t)}\,dt.
\]
Multiplying by $\sqrt d$ yields
\[
d_{\BW}(G_d(\alpha),G_d(\beta))
\le \int_\alpha^\beta\sqrt{dV_d(t)}\,dt
\le \sqrt{dV_{d,\max}(I)}\,|\alpha-\beta|.
\]
Finally, since $\log i\in[0,\log d]$, the variance of the random variable $\log i$ under any probability law is bounded by $(\log d)^2/4$.  Thus $V_{d,\max}(I)\le(\log d)^2/4$, giving the last inequality in \eqref{eq:BW-lipschitz}.
\end{enumerate}
\end{proof}

\begin{remark}[Intrinsic metric consequence]
This identifies the Cartan power-law orbit as a finite-dimensional Fisher--Bures curve.  The phrase does not mean that the Bures--Wasserstein metric is the canonical affine-invariant metric of the symmetric-space realization $\GL(d)/\Orth(d)$.  Rather, $\SPD(d)$ is the quotient state space for singular geometry, and the Bures--Wasserstein metric is the transport metric used on this state space because, on the trace-normalized diagonal orbit, its pullback is exactly $d/4$ times the Fisher information metric.  Consequently, bounds on the exponent coordinate are simultaneously Bures path-length bounds and Fisher path-length bounds along this orbit.
\end{remark}


\section{Main result: slack-aware Cartan-coordinate rigidity of budgeted residual cocycles}
\label{sec:main-rigidity}

The power-law orbit is a one-dimensional chart inside the Cartan chamber.  The local matrix-product budget, however, directly controls only a normalized top-radial functional on the quotient path.  The main result is therefore organized as a deterministic margin statement.  Its local margin is
\[
 b_k:=\log\lambda_k+\eta_k^{\mathrm{nb}},
\]
where $\lambda_k$ is a geometric-mean interface radial amplitude and $\eta_k^{\mathrm{nb}}$ is a measurable non-backtracking slack.  The theorem says that the fitted exponent coordinate moves only when this radial margin or the signed chart residual changes.  Full Cartan-spectral path control is deliberately stated as a separate theorem by adding a full-spectrum Bures/Hellinger chart-residual variation.  This organization prevents the scalar top-radial estimate from being confused with a full-spectrum rigidity statement.

\subsection{Local transport coefficients and radial displacement on the quotient path}

\begin{definition}[Interface geometric-mean radial amplitude]
\label{def:lambda}
For matrices $A,B\in\R^{d\times d}$ with $\|A\|_2,\|B\|_2,\|AB\|_2>0$, define
\[
\Lambda(A,B):=\frac{\|AB\|_2}{\sqrt{\|A\|_2\|B\|_2}}.
\]
For a chain $(W_k)_{k=0}^{L-1}$ write
\[
\lambda_k:=\Lambda(W_{k+1},W_k),\qquad k=0,1,\dots,L-2.
\]
\end{definition}

\begin{remark}[Interpretation and scaling of $\Lambda$]
The coefficient $\Lambda(A,B)$ is a geometric-mean radial amplitude rather than a scale-free angular alignment coefficient.  It compares the two-step norm $\|AB\|_2$ with the geometric mean of the adjacent one-step norms.  The genuinely scale-free submultiplicative ratio would be $\|AB\|_2/(\|A\|_2\|B\|_2)$; the present normalization is used because it directly controls changes in the one-step top radii.  The quantity is scale sensitive: for positive scalars $c,s$,
\[
\Lambda(cA,sB)=\sqrt{cs}\,\Lambda(A,B).
\]
Consequently all budget statements involving $\Lambda$ are made only after a layerwise scale convention, here $\|W_k\|_F^2=d$, has been fixed.  Cancellation, mismatch, and contraction at an interface are recorded jointly by $\Lambda$ and the slack in Definition~\ref{def:nonbacktracking-slack}.
\end{remark}

\begin{definition}[Measurable non-backtracking slack]
\label{def:nonbacktracking-slack}
For an interface $(W_k,W_{k+1})$ define
\[
\eta_k^{\mathrm{nb}}
:=
\left[
\log\frac{\max\{\|W_{k+1}\|_2,\|W_k\|_2\}}{\|W_{k+1}W_k\|_2}
\right]_+,
\qquad [u]_+:=\max\{u,0\}.
\]
Thus $\eta_k^{\mathrm{nb}}=0$ exactly when
\[
\|W_{k+1}W_k\|_2\ge \max\{\|W_{k+1}\|_2,\|W_k\|_2\}.
\]
This condition is a useful clean special case, but it is not imposed in the main theorem.
\end{definition}

\begin{definition}[Combined local radial margin]
\label{def:combined-radial-margin}
For each interface define
\[
 b_k:=\log\lambda_k+\eta_k^{\mathrm{nb}}.
\]
This is the local scalar margin appearing in the main theorem.  Under the hypotheses of Lemma~\ref{lem:slack-radial-bound}, one has $b_k\ge0$.
\end{definition}

\begin{lemma}[Slack-aware radial transport bound]
\label{lem:slack-radial-bound}
Let $P_k:=W_k^\top W_k$ and assume $\|W_k\|_F^2=d$ for all $k$.  Then $b_k\ge0$ and, for every interface,
\begin{equation}\label{eq:radial-local-slack}
 |\rho_d(P_{k+1})-\rho_d(P_k)|
 \le 2b_k.
\end{equation}
Consequently, for every $0\le m<n\le L-1$,
\begin{equation}\label{eq:radial-global-slack}
 |\rho_d(P_n)-\rho_d(P_m)|
 \le 2\sum_{j=m}^{n-1}b_j.
\end{equation}
\end{lemma}

\begin{proof}
Set
\[
 a:=\|W_{k+1}\|_2,
 \qquad
 b:=\|W_k\|_2,
 \qquad
 p:=\|W_{k+1}W_k\|_2,
 \qquad
 m_*:=\max\{a,b\}.
\]
By Definition~\ref{def:lambda}, $p=\lambda_k\sqrt{ab}$.  By Definition~\ref{def:nonbacktracking-slack},
\[
 m_*\le e^{\eta_k^{\mathrm{nb}}}p
 =e^{\eta_k^{\mathrm{nb}}}\lambda_k\sqrt{ab}.
\]
Since $m_*\ge\sqrt{ab}$, this gives $e^{\eta_k^{\mathrm{nb}}}\lambda_k\ge1$, or $b_k\ge0$.  If $a\ge b$, then $m_*=a$ and hence
\[
 \sqrt{a/b}\le e^{\eta_k^{\mathrm{nb}}}\lambda_k,
\]
so $\log(a/b)\le2b_k$.  If $b\ge a$, the same argument gives $\log(b/a)\le2b_k$.  Hence
\[
 |\log a-\log b|\le 2b_k.
\]
On the trace-$d$ slice, $\rho_d(P_k)=\log\|W_k\|_2$, so \eqref{eq:radial-local-slack} follows.  Summing the local inequalities gives \eqref{eq:radial-global-slack}.
\end{proof}

\begin{remark}[Non-backtracking as a zero-slack case]
If $\eta_k^{\mathrm{nb}}=0$ for every interface, Lemma~\ref{lem:slack-radial-bound} reduces to the clean non-backtracking estimate used in the first version of this theorem.  The slack-aware form is preferable because the slack is directly measurable and can be reported together with the fitted spectral coordinates.
\end{remark}

\subsection{Orbit charts: radial residuals and full-spectrum chart error}

\begin{definition}[Radial power-law chart and signed residual]
\label{def:radial-chart-residual}
Let $(P_k)_{k=0}^{L-1}\subset\SPD_d^{(1)}$ and let $\alpha_k>0$.  The signed radial chart residual is
\begin{equation}\label{eq:signed-residual-def}
 r_k:=\rho_d(P_k)-g_d(\alpha_k).
\end{equation}
When $|r_k|\le\delta_k$, the scalar chart error is at most $\delta_k$.  The residual variation along a chain is
\begin{equation}\label{eq:residual-TV-def}
 \TV(r):=\sum_{k=0}^{L-2}|r_{k+1}-r_k|.
\end{equation}
\end{definition}

\begin{definition}[Full-spectrum Cartan chart errors]
\label{def:full-cartan-chart-error}
For $P_k\in\SPD_d^{(1)}$ and a fitted exponent $\alpha_k$, define
\begin{align}
 e_k^{\log}
 &:=
 \bigl\|\mathbf c(P_k)-\mathbf c(G_d(\alpha_k))\bigr\|_\infty,
 \label{eq:elog-def}\\
 e_k^{\BW}
 &:=
 d_{\BW}\bigl(\Cart(P_k),G_d(\alpha_k)\bigr),
 \label{eq:eBW-def}\\
 \bar e_k^{\BW}
 &:={d}^{-1/2}e_k^{\BW}.
 \label{eq:eBW-normalized-def}
\end{align}
The signed residual $r_k$ controls only the normalized top-radial coordinate.  The quantities $e_k^{\log}$ and $e_k^{\BW}$ measure the full spectral distance from the actual Cartan point to the fitted power-law orbit.  The normalized error $\bar e_k^{\BW}$ is dimension-free on the trace-$d$ slice and equals the Hellinger-type root distance between the spectral energy distribution of $P_k$ and the power-law distribution $p^{(\alpha_k)}$.
\end{definition}

\begin{definition}[Explicit fitted chart maps]
\label{def:chart-maps}
Let $I=[\alpha_{\min},\alpha_{\max}]\subset(0,\infty)$ and assume $d\ge2$.  The fitted coordinate in the main theorem should be generated by a specified chart map $\mathcal C_I$; common choices include the following.
\begin{enumerate}[label=(\alph*),leftmargin=2.0em]
\item The \emph{top-radial chart} is
\[
\alpha_I^{\rho}(P):=g_d^{-1}\bigl(\Pi_{g_d(I)}(\rho_d(P))\bigr),
\]
where $\Pi_{g_d(I)}$ denotes Euclidean projection onto the interval $g_d(I)$.  If $\rho_d(P)\in g_d(I)$, this choice gives zero signed radial residual.  It is the sharpest scalar chart but need not coincide with a full-spectrum power-law fit.
\item The \emph{Bures projection chart} is any minimizer
\[
\alpha_I^{\BW}(P)\in\argmin_{\alpha\in I}
 d_{\BW}\bigl(\Cart(P),G_d(\alpha)\bigr).
\]
Existence follows from compactness of $I$ and continuity of the objective.  This choice minimizes the full-spectrum Bures chart error.
\item Given a rank window $\mathcal R\subset\{1,\dots,d\}$, the \emph{log-spectrum least-squares chart} is any minimizer
\[
\alpha_{I,\mathcal R}^{\mathrm{LS}}(W)
\in
\argmin_{\alpha\in I,\ c\in\R}
\sum_{i\in\mathcal R}
\bigl(\log\sigma_i(W)-c+\alpha\log i\bigr)^2.
\]
This is the chart most directly aligned with empirical log--log power-law regression.  Its theorem residuals must be measured unless a separate uniform regression-residual bound is available.
\end{enumerate}
\end{definition}

\begin{remark}[Why the chart map must be specified]
Without an explicit rule for choosing $\alpha_k$, the signed residual $r_k$ can absorb arbitrary chart choices.  The predictive content of the rigidity theorem comes from applying it to a specified chart map and then measuring the residual variation associated with that map.
\end{remark}

\begin{lemma}[From uniform log-window residuals to radial residuals]
\label{lem:logLS-to-radial-residual}
Let $\mathcal R\subset\{1,\ldots,d\}$ be a rank window containing $1$.  Suppose a log-spectrum fit satisfies
\[
\max_{i\in\mathcal R}\left|\log\sigma_i(W)-c+\alpha\log i\right|\le \delta.
\]
If $\|W\|_F^2=d$ and the fitted intercept is normalized so that $c=g_d(\alpha)$, then
\[
 |\rho_d(W^\top W)-g_d(\alpha)|\le \delta.
\]
For ordinary least-squares fits without a uniform residual bound or without $1\in\mathcal R$, the radial residual $r_k$ should be measured directly rather than inferred from the regression loss.
\end{lemma}

\begin{proof}
Since $1\in\mathcal R$, the fit inequality at $i=1$ gives $|\log\sigma_1(W)-c|\le\delta$.  On the trace-$d$ slice, $\rho_d(W^\top W)=\log\|W\|_2=\log\sigma_1(W)$.  With $c=g_d(\alpha)$, the claim follows.
\end{proof}

\begin{definition}[Full-chart residual variations]
\label{def:full-chart-residual-variation}
For a charted path $(P_k,\alpha_k)$, define the Bures square-root residual vector and the log-Cartan residual vector by
\begin{align*}
z_k^{\BW}
&:=
\bigl(\sqrt{\lambda_1(P_k)},\dots,\sqrt{\lambda_d(P_k)}\bigr)
-
\bigl(\sqrt{\lambda_1(\alpha_k)},\dots,\sqrt{\lambda_d(\alpha_k)}\bigr),\\
u_k^{\log}
&:=
\mathbf c(P_k)-\mathbf c(G_d(\alpha_k)),
\end{align*}
where $\lambda_i(\alpha_k)$ denotes the $i$-th diagonal entry of $G_d(\alpha_k)$.  Define
\begin{align}
\TV_{\BW}^{\mathrm{chart}}
&:=\sum_{k=0}^{L-2}\|z_{k+1}^{\BW}-z_k^{\BW}\|_2, \label{eq:TV-BW-chart}\\
\TV_{\log}^{\mathrm{chart}}
&:=\sum_{k=0}^{L-2}\|u_{k+1}^{\log}-u_k^{\log}\|_2. \label{eq:TV-log-chart}
\end{align}
Since $\|z_k^{\BW}\|_2=e_k^{\BW}$ for diagonal Cartan representatives, the endpoint-error bound $\|z_{k+1}^{\BW}-z_k^{\BW}\|_2\le e_{k+1}^{\BW}+e_k^{\BW}$ is always available, but the variation form is sharper when chart errors are coherent across depth.
\end{definition}

\begin{assumption}[Metric power-law chart residual]
\label{ass:metric-chart-residual}
A charted sequence $(P_k,\alpha_k)$ has Bures metric chart residuals $(\varepsilon_k^{\BW})$ if
\[
e_k^{\BW}=d_{\BW}\bigl(\Cart(P_k),G_d(\alpha_k)\bigr)
\le \varepsilon_k^{\BW}.
\]
Equivalently, for diagonal trace-normalized spectra this is a Hellinger-type bound on the corresponding spectral energy measures.  This metric assumption is weaker and more empirical than a uniform multiplicative power-law fit across all singular values.
\end{assumption}

\begin{assumption}[$\delta_{\mathrm{pl}}$-approximate power-law fit]
\label{ass:approx-powerlaw-frob}
Fix $0\le\delta_{\mathrm{pl}}<1$.  Assume each layer $W_k$ satisfies $\|W_k\|_F^2=d$ and there exist parameters $C_k>0$ and $\alpha_k>0$ such that
\[
(1-\delta_{\mathrm{pl}})C_ki^{-\alpha_k}
\le
\sigma_i(W_k)
\le
(1+\delta_{\mathrm{pl}})C_ki^{-\alpha_k},
\qquad i=1,\dots,d.
\]
\end{assumption}

\begin{lemma}[Approximate power laws yield radial and full Cartan charts]
\label{lem:approx-powerlaw-chart}
Under Assumption~\ref{ass:approx-powerlaw-frob}, set
\[
\eta_{\mathrm{pl}}
:=
\log\frac{1+\delta_{\mathrm{pl}}}{1-\delta_{\mathrm{pl}}}.
\]
Then
\begin{align}
 |r_k|=|\rho_d(P_k)-g_d(\alpha_k)|&\le \eta_{\mathrm{pl}}, \label{eq:approx-radial-chart}\\
 e_k^{\log}&\le 2\eta_{\mathrm{pl}}, \label{eq:approx-log-chart}\\
 e_k^{\BW}&\le \sqrt d\,(e^{\eta_{\mathrm{pl}}}-1), \label{eq:approx-BW-chart}\\
 \bar e_k^{\BW}&\le e^{\eta_{\mathrm{pl}}}-1. \label{eq:approx-normalized-BW-chart}
\end{align}
\end{lemma}

\begin{proof}
The radial estimate is the same normalization argument as in Proposition~\ref{prop:cartan-membership}.  From the fit inequalities and $\|W_k\|_F^2=d$,
\[
(1-\delta_{\mathrm{pl}})^2 C_k^2\Harm_{d,2\alpha_k}
\le d\le
(1+\delta_{\mathrm{pl}})^2 C_k^2\Harm_{d,2\alpha_k}.
\]
Combining this with the corresponding bound for the top singular value gives
\[
\sqrt{\frac d{\Harm_{d,2\alpha_k}}}
\frac{1-\delta_{\mathrm{pl}}}{1+\delta_{\mathrm{pl}}}
\le
\|W_k\|_2
\le
\sqrt{\frac d{\Harm_{d,2\alpha_k}}}
\frac{1+\delta_{\mathrm{pl}}}{1-\delta_{\mathrm{pl}}},
\]
which proves \eqref{eq:approx-radial-chart}.  For the full logarithmic Cartan vector, the same inequalities give for every $i$
\[
\left|
\log\sigma_i(W_k)^2
-
\log\left(\frac d{\Harm_{d,2\alpha_k}}i^{-2\alpha_k}\right)
\right|
\le 2\eta_{\mathrm{pl}},
\]
which is \eqref{eq:approx-log-chart}.  Finally, write the diagonal entries of $\Cart(P_k)$ as $\lambda_i(G_d(\alpha_k))e^{\Delta_i}$ with $|\Delta_i|\le2\eta_{\mathrm{pl}}$.  Since both matrices are diagonal, the Bures distance is the Euclidean distance between their entrywise square roots, hence
\[
 e_k^{\BW}
 \le
 (e^{\eta_{\mathrm{pl}}}-1)
 \left(\sum_{i=1}^d\lambda_i(G_d(\alpha_k))\right)^{1/2}
 =
 \sqrt d\,(e^{\eta_{\mathrm{pl}}}-1).
\]
Dividing by $\sqrt d$ gives \eqref{eq:approx-normalized-BW-chart}.
\end{proof}

\subsection{Top-radial coordinate rigidity}

\begin{theorem}[Slack-aware top-radial Cartan-coordinate rigidity]
\label{thm:cartan-rigidity}
Assume $d\ge2$ and let $(W_k)_{k=0}^{L-1}$ be a chain of full-rank matrices in $\R^{d\times d}$ with $L\ge2$.  Set $P_k:=W_k^\top W_k$ and assume $\|W_k\|_F^2=d$ for all $k$.  Let $\lambda_k=\Lambda(W_{k+1},W_k)$, let $\eta_k^{\mathrm{nb}}$ be the slack in Definition~\ref{def:nonbacktracking-slack}, and set $b_k=\log\lambda_k+\eta_k^{\mathrm{nb}}$.  Fix a chart map $\mathcal C_I$ as in Definition~\ref{def:chart-maps}, or any other specified chart rule, and set $\alpha_k=\mathcal C_I(P_k)\in I=[\alpha_{\min},\alpha_{\max}]\subset(0,\infty)$.  Define the signed residuals $r_k=\rho_d(P_k)-g_d(\alpha_k)$.  Let
\[
 m_d(I):=\min_{\alpha\in I}g_d'(\alpha)>0.
\]
Then the following scalar top-radial conclusions hold.

\begin{enumerate}[label=(C\arabic*),leftmargin=2.2em]
\item \textbf{Radial shortness.}
For every $0\le m<n\le L-1$,
\begin{equation}\label{eq:radial-shortness-slack-main}
 |\rho_d(P_n)-\rho_d(P_m)|
 \le
 2\sum_{j=m}^{n-1}b_j.
\end{equation}

\item \textbf{Coordinate rigidity with signed residual variation.}
For every $0\le m<n\le L-1$, there is a point $\xi_{m,n}$ between $\alpha_m$ and $\alpha_n$ such that
\begin{equation}\label{eq:alpha-mn-slack-signed}
 |\alpha_n-\alpha_m|
 \le
 \frac{2\sum_{j=m}^{n-1}b_j+|r_n-r_m|}{g_d'(\xi_{m,n})}
 \le
 \frac{2\sum_{j=m}^{n-1}b_j+|r_n-r_m|}{m_d(I)}.
\end{equation}
In particular,
\begin{equation}\label{eq:alpha-local-exact}
 |\alpha_{k+1}-\alpha_k|
 \le
 \frac{2b_k+|r_{k+1}-r_k|}{m_d(I)}.
\end{equation}
If only unsigned chart errors $|r_k|\le\delta_k$ are known, then
\begin{equation}\label{eq:alpha-local-delta}
 |\alpha_{k+1}-\alpha_k|
 \le
 \frac{2b_k+
 \delta_{k+1}+\delta_k}{m_d(I)}.
\end{equation}

\item \textbf{Total variation bound.}
The fitted exponent profile satisfies
\begin{equation}\label{eq:TV-slack-signed}
 \sum_{k=0}^{L-2}|\alpha_{k+1}-\alpha_k|
 \le
 \frac{2\sum_{k=0}^{L-2}b_k+
 \TV(r)}{m_d(I)}.
\end{equation}
If $|r_k|\le\bar\delta$ for all $k$, then the cruder but sometimes convenient bound
\begin{equation}\label{eq:TV-slack-crude-delta}
 \sum_{k=0}^{L-2}|\alpha_{k+1}-\alpha_k|
 \le
 \frac{2\sum_{k=0}^{L-2}b_k
 +2(L-1)\bar\delta}{m_d(I)}
\end{equation}
holds.

\item \textbf{Small-margin specialization.}
If $b_k\le B/L$ and $|r_{k+1}-r_k|\le R/L$ for every $k$, then
\begin{equation}\label{eq:small-margin-local}
 \max_{0\le k\le L-2}|\alpha_{k+1}-\alpha_k|
 \le
 \frac{2B+R}{L\,m_d(I)}.
\end{equation}
If $M\ge1$, $\lambda_k\le M^{2/L}$, $\eta_k^{\mathrm{nb}}=0$, and $r_{k+1}=r_k$ for every $k$, then this gives the zero-slack exact-chart estimate
\begin{equation}\label{eq:uniform-local-slack}
 \max_{0\le k\le L-2}|\alpha_{k+1}-\alpha_k|
 \le
 \frac{4\log M}{L\,m_d(I)}.
\end{equation}
More generally, if only $\lambda_k\le M^{2/L}$ is known, then $b_k\le 2\log M/L+\eta_k^{\mathrm{nb}}$ and
\begin{equation}\label{eq:uniform-TV-slack}
 \sum_{k=0}^{L-2}|\alpha_{k+1}-\alpha_k|
 \le
 \frac{4\log M+2\sum_{k=0}^{L-2}\eta_k^{\mathrm{nb}}+\TV(r)}{m_d(I)}.
\end{equation}
\end{enumerate}
\end{theorem}

\begin{proof}
(C1) is exactly Lemma~\ref{lem:slack-radial-bound}.  For (C2), use the signed residual identity
\[
 g_d(\alpha_n)-g_d(\alpha_m)
 =
 \rho_d(P_n)-\rho_d(P_m)-(r_n-r_m).
\]
Taking absolute values and applying (C1) gives
\[
 |g_d(\alpha_n)-g_d(\alpha_m)|
 \le 2\sum_{j=m}^{n-1}b_j+|r_n-r_m|.
\]
Since $g_d$ is continuously differentiable and strictly increasing on $I$, the mean-value theorem gives
\[
 |g_d(\alpha_n)-g_d(\alpha_m)|
 =g_d'(\xi_{m,n})|\alpha_n-\alpha_m|
\]
for some $\xi_{m,n}\in I$, proving \eqref{eq:alpha-mn-slack-signed}.  The local estimate \eqref{eq:alpha-local-exact} is the case $n=m+1$, and \eqref{eq:alpha-local-delta} follows from $|r_{k+1}-r_k|\le\delta_{k+1}+\delta_k$.

For (C3), sum \eqref{eq:alpha-local-exact} over $k$.  The crude version follows from $\TV(r)\le2(L-1)\bar\delta$.  For (C4), substitute the stated local bounds into \eqref{eq:alpha-local-exact}; the uniform-budget versions use $\log\lambda_k\le2\log M/L$.
\end{proof}

\subsection{Fitted-orbit and actual-Cartan path bounds}

\begin{theorem}[Fitted and actual Cartan path length]
\label{thm:fitted-actual-cartan-path}
Assume the hypotheses of Theorem~\ref{thm:cartan-rigidity} and set
\[
 V_{d,\max}(I):=\max_{\alpha\in I}V_d(\alpha),
 \qquad
 \widehat P_k:=G_d(\alpha_k).
\]
Then the fitted orbit path satisfies
\begin{equation}\label{eq:fitted-BW-slack}
 \sum_{k=0}^{L-2}d_{\BW}(\widehat P_{k+1},\widehat P_k)
 \le
 \sqrt{dV_{d,\max}(I)}
 \sum_{k=0}^{L-2}|\alpha_{k+1}-\alpha_k|.
\end{equation}
Consequently,
\begin{equation}\label{eq:fitted-BW-slack-expanded}
 \sum_{k=0}^{L-2}d_{\BW}(\widehat P_{k+1},\widehat P_k)
 \le
 \frac{\sqrt{dV_{d,\max}(I)}}{m_d(I)}
 \left(
 2\sum_{k=0}^{L-2}b_k
 +\TV(r)
 \right).
\end{equation}
If the Bures square-root residual variation from Definition~\ref{def:full-chart-residual-variation} is available, then the actual Cartan representatives satisfy
\begin{equation}\label{eq:actual-cartan-BW-variation}
 \sum_{k=0}^{L-2}d_{\BW}\bigl(\Cart(P_{k+1}),\Cart(P_k)\bigr)
 \le
 \TV_{\BW}^{\mathrm{chart}}
 +
 \sum_{k=0}^{L-2}d_{\BW}(\widehat P_{k+1},\widehat P_k).
\end{equation}
Consequently, the coarser endpoint-error bound
\begin{equation}\label{eq:actual-cartan-BW-slack}
 \sum_{k=0}^{L-2}d_{\BW}\bigl(\Cart(P_{k+1}),\Cart(P_k)\bigr)
 \le
 \sum_{k=0}^{L-2}(e_{k+1}^{\BW}+e_k^{\BW})
 +
 \sum_{k=0}^{L-2}d_{\BW}(\widehat P_{k+1},\widehat P_k)
\end{equation}
also holds.  Equivalently, the normalized path length $d^{-1/2}\sum_k d_{\BW}(\cdot,\cdot)$ is controlled by the Hellinger-scale chart variation $d^{-1/2}\TV_{\BW}^{\mathrm{chart}}$ plus the fitted Fisher--Bures orbit length.
\end{theorem}

\begin{proof}
Theorem~\ref{thm:fisher-bures}(F3) gives
\[
 d_{\BW}(G_d(\alpha_{k+1}),G_d(\alpha_k))
 \le
 \sqrt{dV_{d,\max}(I)}|\alpha_{k+1}-\alpha_k|.
\]
Summing proves \eqref{eq:fitted-BW-slack}; inserting \eqref{eq:TV-slack-signed} proves \eqref{eq:fitted-BW-slack-expanded}.  For the sharper actual-path bound, use the diagonal Bures square-root representation.  The square-root vector of $\Cart(P_k)$ is the square-root vector of $\widehat P_k$ plus $z_k^{\BW}$.  Therefore
\[
 d_{\BW}(\Cart(P_{k+1}),\Cart(P_k))
 \le
 d_{\BW}(\widehat P_{k+1},\widehat P_k)
 +\|z_{k+1}^{\BW}-z_k^{\BW}\|_2.
\]
Summing gives \eqref{eq:actual-cartan-BW-variation}.  Since $\|z_{k+1}^{\BW}-z_k^{\BW}\|_2\le e_{k+1}^{\BW}+e_k^{\BW}$, \eqref{eq:actual-cartan-BW-slack} follows.
\end{proof}

\begin{remark}[Log-Cartan residual variation]
The same residual-variation principle can be stated in logarithmic Cartan coordinates.  With $u_k^{\log}$ from Definition~\ref{def:full-chart-residual-variation},
\[
\sum_{k=0}^{L-2}\|\mathbf c(P_{k+1})-\mathbf c(P_k)\|_2
\le
\sum_{k=0}^{L-2}\|\mathbf c(G_d(\alpha_{k+1}))-\mathbf c(G_d(\alpha_k))\|_2
+\TV_{\log}^{\mathrm{chart}}.
\]
Thus actual full-Cartan shortness is controlled by fitted-orbit shortness plus variation of the full-spectrum chart residual, not merely by absolute per-layer chart error.
\end{remark}

\begin{remark}[What the rigidity theorems control]
Theorem~\ref{thm:cartan-rigidity} controls a scalar normalized top-radial coordinate.  Theorem~\ref{thm:fitted-actual-cartan-path} controls the full Cartan spectral path only after the full-spectrum chart residual variation is included.  This distinction is essential in applications: a smooth fitted exponent path is meaningful only on layers where the power-law chart residuals are small.
\end{remark}

\subsection{Residual near-identity scaling and transport margins}

\begin{proposition}[Residual near-identity expansion of the interface budget]
\label{prop:near-identity-interface-expansion}
Let $t=1/L$ and suppose
\[
J_k=\Id+tA_k,
\qquad
\|A_k\|_2\le C,
\]
with $tC<1$.  Write
\[
S_k:=\frac12(A_k+A_k^\top),
\qquad
\mu_k:=\lambda_{\max}(S_k),
\qquad
\mu_{k+1,k}:=\lambda_{\max}(S_{k+1}+S_k).
\]
Then, uniformly over bounded generators,
\begin{align}
\log\|J_k\|_2&=t\mu_k+O_C(t^2), \label{eq:nearid-single-norm}\\
\log\|J_{k+1}J_k\|_2&=t\mu_{k+1,k}+O_C(t^2), \label{eq:nearid-product-norm}\\
\log\Lambda(J_{k+1},J_k)
&=
 t\left(\mu_{k+1,k}-\frac12\mu_{k+1}-\frac12\mu_k\right)+O_C(t^2), \label{eq:nearid-lambda}\\
\eta_k^{\mathrm{nb}}
&=
 t\left[\max\{\mu_{k+1},\mu_k\}-\mu_{k+1,k}\right]_+
 +O_C(t^2). \label{eq:nearid-slack}
\end{align}
In particular, for depth-scaled residual factors with uniformly bounded generators, the local radial margin $b_k=\log\lambda_k+\eta_k^{\mathrm{nb}}$ is $O_C(1/L)$.
\end{proposition}

\begin{proof}
For $J_k=\Id+tA_k$,
\[
J_k^\top J_k
=
\Id+2tS_k+t^2A_k^\top A_k.
\]
Weyl's inequality gives
\[
\lambda_{\max}(J_k^\top J_k)
=
1+2t\mu_k+O_C(t^2),
\]
and taking the square root and logarithm gives \eqref{eq:nearid-single-norm}.  Similarly,
\[
J_{k+1}J_k
=
\Id+t(A_{k+1}+A_k)+O_C(t^2),
\]
so the symmetric part of the first-order perturbation is $S_{k+1}+S_k$, and the same argument gives \eqref{eq:nearid-product-norm}.  Since
\[
\log\Lambda(J_{k+1},J_k)
=
\log\|J_{k+1}J_k\|_2
-\frac12\log\|J_{k+1}\|_2
-\frac12\log\|J_k\|_2,
\]
substitution gives \eqref{eq:nearid-lambda}.  Finally,
\[
\eta_k^{\mathrm{nb}}
=
\left[\max\{\log\|J_{k+1}\|_2,\log\|J_k\|_2\}
-
\log\|J_{k+1}J_k\|_2\right]_+,
\]
and \eqref{eq:nearid-slack} follows from the preceding expansions and the Lipschitz continuity of $u\mapsto[u]_+$.
\end{proof}

\begin{proposition}[Frobenius-normalized near-identity bridge]
\label{prop:normalized-near-identity-bridge}
Let $J_k=\Id+tA_k$ with $\|A_k\|_2\le C$ and $tC<1$, and define the Frobenius-normalized factors
\[
W_k:=\sqrt d\,\frac{J_k}{\|J_k\|_F}.
\]
With $S_k=(A_k+A_k^\top)/2$, set
\[
\tau_k:=d^{-1}\tr S_k,
\qquad
\mu_k:=\lambda_{\max}(S_k),
\qquad
\mu_{k+1,k}:=\lambda_{\max}(S_{k+1}+S_k).
\]
Then, uniformly for bounded generators,
\begin{align}
\log\|W_k\|_2
&=t(\mu_k-\tau_k)+O_C(t^2),\label{eq:normalized-nearid-single}\\
\log\|W_{k+1}W_k\|_2
&=t(\mu_{k+1,k}-\tau_{k+1}-\tau_k)+O_C(t^2),\label{eq:normalized-nearid-product}\\
\log\Lambda(W_{k+1},W_k)
&=t\left(\mu_{k+1,k}-\frac12\mu_{k+1}-\frac12\mu_k-\frac12\tau_{k+1}-\frac12\tau_k\right)+O_C(t^2),\label{eq:normalized-nearid-lambda}\\
\eta_k^{\mathrm{nb}}(W)
&=t\left[\max\{\mu_{k+1}-\tau_{k+1},\mu_k-\tau_k\}-(\mu_{k+1,k}-\tau_{k+1}-\tau_k)\right]_++O_C(t^2).
\label{eq:normalized-nearid-slack}
\end{align}
In particular,
\begin{equation}\label{eq:normalized-transport-Ot}
 b_k(W)=\log\Lambda(W_{k+1},W_k)+\eta_k^{\mathrm{nb}}(W)=O_C(t).
\end{equation}
Thus the transport side of Theorem~\ref{thm:cartan-rigidity} is naturally $O_C(1/L)$ for Frobenius-normalized depth-scaled residual factors.
\end{proposition}

\begin{proof}
Since
\[
\|J_k\|_F^2
=\tr(J_k^\top J_k)
=d+2t\tr S_k+O_C(dt^2),
\]
we have
\[
\log\left(\sqrt d/\|J_k\|_F\right)=-t\tau_k+O_C(t^2).
\]
Combining this scale correction with \eqref{eq:nearid-single-norm} gives \eqref{eq:normalized-nearid-single}.  Applying the same scale correction to the product
$W_{k+1}W_k=(\sqrt d/\|J_{k+1}\|_F)(\sqrt d/\|J_k\|_F)J_{k+1}J_k$
and using \eqref{eq:nearid-product-norm} gives \eqref{eq:normalized-nearid-product}.  Substituting \eqref{eq:normalized-nearid-single} and \eqref{eq:normalized-nearid-product} into the definition of $\log\Lambda$ gives \eqref{eq:normalized-nearid-lambda}.  The slack expansion follows from the definition of $\eta_k^{\mathrm{nb}}$ and the Lipschitz continuity of $[\cdot]_+$.  The $O_C(t)$ bound in \eqref{eq:normalized-transport-Ot} follows immediately.
\end{proof}

\begin{corollary}[A concrete bounded-generator margin]
\label{cor:explicit-nearid-margin}
Assume the setting of Proposition~\ref{prop:normalized-near-identity-bridge} and suppose $tC\le1/4$.  Then the Frobenius-normalized factors satisfy the explicit coarse bound
\[
 b_k(W)\le 18Ct,
 \qquad
 2b_k(W)\le 36Ct.
\]
Consequently, if in addition $|r_{k+1}-r_k|\le Rt$, then
\[
 |\alpha_{k+1}-\alpha_k|
 \le
 \frac{(36C+R)t}{m_d(I)}.
\]
\end{corollary}

\begin{proof}
Since the singular values of $J_k=I+tA_k$ lie in $[1-tC,1+tC]$, those of $W_k=\sqrt d J_k/\|J_k\|_F$ lie in $[(1-tC)/(1+tC),(1+tC)/(1-tC)]$.  For $tC\le1/4$, this gives $|\log\|W_k\|_2|\le3Ct$ and $|\log\|W_{k+1}W_k\|_2|\le6Ct$.  Hence $|\log\Lambda(W_{k+1},W_k)|\le9Ct$ and $\eta_k^{\mathrm{nb}}(W)\le9Ct$, so $b_k(W)\le18Ct$.  The last estimate is Theorem~\ref{thm:cartan-rigidity}(C2).
\end{proof}

\subsection{Compatibility of residual scaling and power-law chart quality}

\begin{remark}[Transport margins and chart quality are separate]
Propositions~\ref{prop:near-identity-interface-expansion} and~\ref{prop:normalized-near-identity-bridge} verify only the transport-budget side of the rigidity theorem.  They do not assert power-law chart quality, nor do they imply that $|r_{k+1}-r_k|$, $e_k^{\BW}$, or $\TV_{\BW}^{\mathrm{chart}}$ is small.  Conversely, smooth or heavy-tailed spectra of static weight matrices do not by themselves verify the residual Jacobian transport margin unless those matrices are the same operators used in the chain.  A full theorem margin test must measure both sides: the local radial margins $b_k$ and the chart residuals associated with the chosen chart map.
\end{remark}

\begin{remark}[Static weights versus residual Jacobians]
For an actual residual block, the Jacobian is $J_k(x)=I+DF_k(x)$ and may include attention softmax derivatives, normalization, activation derivatives, and skip paths.  A static weight matrix from a convolution, attention projection, or MLP layer is therefore a spectral diagnostic rather than an automatic instance of the residual Jacobian chain theorem.  To use the deterministic theorem directly, one should either compute or approximate the block Jacobians on data, or state a separate linearization argument connecting the static matrices to the operator chain under study.
\end{remark}

\begin{corollary}[Exact power-law specialization and the zero-slack non-backtracking case]
\label{cor:gsa-spectral-domain}
Assume the setup of Theorem~\ref{thm:cartan-rigidity} and exact power-law singular values as in Assumption~\ref{ass:power-law}.  Then $r_k=0$ and $e_k^{\BW}=0$ for every $k$.  Hence
\begin{equation}\label{eq:exact-slack-local}
 |\alpha_{k+1}-\alpha_k|
 \le
 \frac{2b_k}{m_d(I)}.
\end{equation}
If, in addition, the non-backtracking slack vanishes and $\lambda_k\le M^{2/L}$, then
\begin{equation}\label{eq:exact-uniform-alpha}
 \max_{0\le k\le L-2}|\alpha_{k+1}-\alpha_k|
 \le
 \frac{4\log M}{L\,m_d(I)}.
\end{equation}
Moreover, since $C_k=\|W_k\|_2$ in the exact model,
\begin{equation}\label{eq:gsa-scale-drift}
\left|\log\frac{C_{k+1}}{C_k}\right|
\le
2b_k,
\end{equation}
and in the zero-slack uniform-budget case this is at most $4\log M/L$.
\end{corollary}

\begin{proof}
Exact orbit membership gives $\Cart(P_k)=G_d(\alpha_k)$ by Proposition~\ref{prop:cartan-membership}; hence $\rho_d(P_k)=g_d(\alpha_k)$ and $r_k=0$, while the full Bures chart error is also zero.  The local estimate \eqref{eq:exact-slack-local} follows from \eqref{eq:alpha-local-exact}.  The uniform zero-slack estimate follows from $\log\lambda_k\le2\log M/L$.  Finally $C_k=\sigma_1(W_k)=\|W_k\|_2$, so \eqref{eq:gsa-scale-drift} is exactly Lemma~\ref{lem:slack-radial-bound} written for the exact power-law scale.
\end{proof}

\begin{theorem}[Robust rigidity under approximate and metric power-law charts]
\label{thm:robust-cartan-rigidity}
Assume the setup of Theorem~\ref{thm:cartan-rigidity}.  The measured-residual version is simply
\begin{equation}\label{eq:robust-measured-residual}
 |\alpha_{k+1}-\alpha_k|
 \le
 \frac{2b_k+|r_{k+1}-r_k|}{m_d(I)},
\end{equation}
where $r_k=\rho_d(P_k)-g_d(\alpha_k)$ is computed from the chosen chart map.  If, in addition, the strong multiplicative fit in Assumption~\ref{ass:approx-powerlaw-frob} holds, set
\[
\eta_{\mathrm{pl}}:=\log\frac{1+\delta_{\mathrm{pl}}}{1-\delta_{\mathrm{pl}}}.
\]
Then the worst-case bound
\begin{equation}\label{eq:robust-local-slack}
 |\alpha_{k+1}-\alpha_k|
 \le
 \frac{2b_k+2\eta_{\mathrm{pl}}}{m_d(I)}
\end{equation}
holds.  Under the zero-slack uniform budget $M\ge1$ and $\lambda_k\le M^{2/L}$,
\begin{equation}\label{eq:robust-uniform-slack}
 \max_{0\le k\le L-2}|\alpha_{k+1}-\alpha_k|
 \le
 \frac{4\log M/L+2\eta_{\mathrm{pl}}}{m_d(I)}.
\end{equation}
Furthermore, if only metric chart information is assumed, namely Assumption~\ref{ass:metric-chart-residual} or the sharper variation $\TV_{\BW}^{\mathrm{chart}}$, the actual Cartan path is controlled by Theorem~\ref{thm:fitted-actual-cartan-path}.  Under the strong multiplicative fit,
\[
 e_k^{\BW}\le \sqrt d\,(e^{\eta_{\mathrm{pl}}}-1),
 \qquad
 \bar e_k^{\BW}\le e^{\eta_{\mathrm{pl}}}-1.
\]
\end{theorem}

\begin{proof}
The measured-residual estimate is Theorem~\ref{thm:cartan-rigidity}(C2).  Under Assumption~\ref{ass:approx-powerlaw-frob}, Lemma~\ref{lem:approx-powerlaw-chart} gives $|r_k|\le\eta_{\mathrm{pl}}$ and hence $|r_{k+1}-r_k|\le2\eta_{\mathrm{pl}}$, which proves \eqref{eq:robust-local-slack}.  The uniform statement follows from $b_k=\log\lambda_k\le2\log M/L$ in the zero-slack case.  The Bures chart-error statement is Theorem~\ref{thm:fitted-actual-cartan-path} together with Lemma~\ref{lem:approx-powerlaw-chart} when the strong multiplicative fit is used.
\end{proof}

\begin{theorem}[Budget converse and spectral-shock lower bound]
\label{thm:budget-converse}
Assume the hypotheses of Theorem~\ref{thm:cartan-rigidity}.  Then every interface satisfies
\begin{equation}\label{eq:converse-slack-signed}
 b_k
 \ge
 \frac12\left(
 |g_d(\alpha_{k+1})-g_d(\alpha_k)|
 -|r_{k+1}-r_k|
 \right)_+.
\end{equation}
If $|r_k|\le\delta_k$, then
\begin{equation}\label{eq:converse-slack-delta}
 b_k
 \ge
 \frac12\left(
 |g_d(\alpha_{k+1})-g_d(\alpha_k)|
 -\delta_{k+1}-\delta_k
 \right)_+.
\end{equation}
In particular, in the exact zero-slack chart,
\begin{equation}\label{eq:converse-exact-zero-slack}
 \log\lambda_k
 \ge
 \frac12|g_d(\alpha_{k+1})-g_d(\alpha_k)|
 \ge
 \frac12 m_d(I)|\alpha_{k+1}-\alpha_k|.
\end{equation}
\end{theorem}

\begin{proof}
The signed residual identity gives
\[
 |g_d(\alpha_{k+1})-g_d(\alpha_k)|
 \le
 |\rho_d(P_{k+1})-\rho_d(P_k)|+|r_{k+1}-r_k|.
\]
Lemma~\ref{lem:slack-radial-bound} bounds the radial term by $2b_k$, which proves \eqref{eq:converse-slack-signed}.  The unsigned-residual version follows from $|r_{k+1}-r_k|\le\delta_{k+1}+\delta_k$.  The exact zero-slack case sets $r_k=0$ and $\eta_k^{\mathrm{nb}}=0$; the final inequality follows from the mean-value theorem and the definition of $m_d(I)$.
\end{proof}

\begin{remark}[Empirical interpretation of the converse]
The forward theorem says that small interface margins and small residual variation force a short fitted spectral-coordinate path.  The converse says that a measured spectral shock cannot occur without at least one of these quantities becoming large.  Thus a local empirical margin test should report $\widehat\Lambda_k$, $\widehat\eta_k^{\mathrm{nb}}$, the combined margin $\widehat b_k$, and the chart residual increment $|\widehat r_{k+1}-\widehat r_k|$ together with $|\widehat\alpha_{k+1}-\widehat\alpha_k|$.
\end{remark}

\begin{theorem}[Fisher--KL and Bures action bound for a budgeted chain]
\label{thm:fisher-kl-bures-action}
Assume the hypotheses of Theorem~\ref{thm:cartan-rigidity} on the interval $I$, and set
\[
B_k:=2b_k+|r_{k+1}-r_k|,
\qquad
I_{d,\max}(I):=\max_{\alpha\in I}I_d(\alpha).
\]
Then
\begin{equation}\label{eq:KL-action-bound}
\sum_{k=0}^{L-2}\KL_d(\alpha_k\|\alpha_{k+1})
\le
\frac{I_{d,\max}(I)}{2m_d(I)^2}
\sum_{k=0}^{L-2}B_k^2.
\end{equation}
Moreover,
\begin{equation}\label{eq:BW-action-bound}
\sum_{k=0}^{L-2}d_{\BW}(G_d(\alpha_{k+1}),G_d(\alpha_k))^2
\le
\frac{dV_{d,\max}(I)}{m_d(I)^2}
\sum_{k=0}^{L-2}B_k^2.
\end{equation}
In the exact zero-slack uniform-budget case, $B_k\le4\log M/L$, so
\begin{align}
\sum_{k=0}^{L-2}\KL_d(\alpha_k\|\alpha_{k+1})
&\le
\frac{8(L-1)I_{d,\max}(I)}{m_d(I)^2}
\left(\frac{\log M}{L}\right)^2, \\
\sum_{k=0}^{L-2}d_{\BW}(G_d(\alpha_{k+1}),G_d(\alpha_k))^2
&\le
\frac{16(L-1)dV_{d,\max}(I)}{m_d(I)^2}
\left(\frac{\log M}{L}\right)^2.
\end{align}
\end{theorem}

\begin{proof}
The local coordinate estimate \eqref{eq:alpha-local-exact} gives $|\alpha_{k+1}-\alpha_k|\le B_k/m_d(I)$.  Proposition~\ref{prop:entropy-KL-Fisher} gives
\[
\KL_d(\alpha_k\|\alpha_{k+1})
=\frac12I_d(\xi_k)(\alpha_{k+1}-\alpha_k)^2
\]
for some $\xi_k\in I$, proving \eqref{eq:KL-action-bound}.  The Bures action estimate follows from Theorem~\ref{thm:fisher-bures}(F3) after squaring the local bound.  The exact zero-slack uniform-budget constants follow from $B_k=2\log\lambda_k\le4\log M/L$.
\end{proof}

\section{Spectral tail geometry and compressibility}
\label{sec:compressibility}

The compressibility part of the theory is best formulated as a spectral-tail quantity rather than only as a rank statistic but as a tail problem for a probability measure on the rank set.
Along the Cartan power-law orbit, this measure is exactly the Gibbs family already introduced above.
The effective rank is therefore an energy-truncation quantile of a spectral tail measure.

\begin{definition}[Spectral energy measure and truncation rank]
\label{def:effective-rank}
Let $W\in\R^{d\times d}$ have singular values $\sigma_1(W)\ge\cdots\ge\sigma_d(W)\ge 0$.
Define the spectral energy measure
\[
\mu_W:=\sum_{i=1}^d \frac{\sigma_i(W)^2}{\|W\|_F^2}\,\delta_i,
\]
a probability measure on $\{1,\dots,d\}$.
For $0<\eps<1$, define the truncation rank
\[
R_\eps(W):=
\min\Bigl\{r\in\{1,\dots,d\}:\mu_W(\{1,\dots,r\})\ge 1-\eps\Bigr\}.
\]
Equivalently,
\[
R_\eps(W)=
\min\left\{r:\sum_{i=1}^r\sigma_i(W)^2\ge (1-\eps)\sum_{i=1}^d \sigma_i(W)^2\right\}.
\]
\end{definition}

\begin{remark}[Geometric interpretation]
The quantity $R_\eps(W)$ is the smallest spectral truncation at which the projected Gram point $W^\top W$ retains at least $(1-\eps)$ of its total energy.
It is therefore a quantile of a spectral measure rather than a combinatorial surrogate for rank.
\end{remark}

\begin{definition}[Power-law tail measure on the Gibbs--Cartan orbit]
\label{def:tail-measure}
Under the exact power-law model $\sigma_i(W)=Ci^{-\alpha}$, define
\[
\nu_\alpha:=\sum_{i=1}^d \frac{i^{-2\alpha}}{\Harm_{d,2\alpha}}\,\delta_i.
\]
Then $\mu_W=\nu_\alpha$.
For $r\in\{0,1,\dots,d\}$, define the tail function
\[
\tau_\alpha(r):=\nu_\alpha(\{r+1,\dots,d\})
=
\frac{\Harm_{d,2\alpha}-\Harm_{r,2\alpha}}{\Harm_{d,2\alpha}},
\]
with the convention $\Harm_{0,s}:=0$; thus $\tau_\alpha(0)=1$ and $\tau_\alpha(d)=0$.
\end{definition}

\begin{lemma}[Integral bounds for power-law tails]
\label{lem:integral-test}
Let $s>0$ and $1\le r<d$. Then
\begin{equation}\label{eq:integral-test}
\int_{r+1}^{d+1} x^{-s}\,dx
\le
\sum_{i=r+1}^{d} i^{-s}
\le
\int_{r}^{d} x^{-s}\,dx.
\end{equation}
If $s>1$, then
\begin{equation}\label{eq:tail-upper}
\sum_{i=r+1}^{d} i^{-s}
\le
\int_{r}^{\infty} x^{-s}\,dx
=
\frac{r^{1-s}}{s-1}.
\end{equation}
\end{lemma}

\begin{proof}
Let $f(x):=x^{-s}$.  For $s>0$, $f$ is positive and decreasing on $[1,\infty)$.  For every integer $i\ge 1$, monotonicity gives
\[
\int_i^{i+1} f(x)\,dx\le f(i)\le \int_{i-1}^{i}f(x)\,dx.
\]
Apply the left inequality with $i=r+1,\ldots,d$ and sum to obtain
\[
\int_{r+1}^{d+1}x^{-s}\,dx
=\sum_{i=r+1}^{d}\int_i^{i+1}x^{-s}\,dx
\le\sum_{i=r+1}^{d}i^{-s}.
\]
Apply the right inequality with the same indices and sum to obtain
\[
\sum_{i=r+1}^{d}i^{-s}
\le\sum_{i=r+1}^{d}\int_{i-1}^{i}x^{-s}\,dx
=\int_r^d x^{-s}\,dx.
\]
This proves \eqref{eq:integral-test}.  If $s>1$, the improper integral converges, and the already established upper bound gives
\[
\sum_{i=r+1}^{d}i^{-s}
\le\int_r^d x^{-s}\,dx
\le\int_r^\infty x^{-s}\,dx
=\left[\frac{x^{1-s}}{1-s}\right]_{x=r}^{\infty}
=\frac{r^{1-s}}{s-1}.
\]
This is \eqref{eq:tail-upper}.
\end{proof}

\begin{theorem}[Energy truncation on the Gibbs--Cartan tail]
\label{thm:effective-rank}
Let $W\in\R^{d\times d}$ satisfy the exact power-law model
\[
\sigma_i(W)=C\,i^{-\alpha},\qquad \alpha>\frac12.
\]
Then $\mu_W=\nu_\alpha$, and the truncation rank from Definition~\ref{def:effective-rank} is
\[
R_\eps(W)=\min\{r:\tau_\alpha(r)\le \eps\}.
\]
Moreover, for every $r\in\{1,\dots,d-1\}$,
\begin{equation}\label{eq:tail-condition}
\Harm_{d,2\alpha}-\Harm_{r,2\alpha}
\le
\frac{r^{1-2\alpha}}{2\alpha-1},
\end{equation}
hence
\begin{equation}\label{eq:Reps-upper}
R_\eps(W)
\le
\min\left\{
 d,\;
\left\lceil
\left(\frac{1}{(2\alpha-1)\,\eps\,\Harm_{d,2\alpha}}\right)^{\frac{1}{2\alpha-1}}
\right\rceil
\right\}.
\end{equation}
Conversely, if $r\in\{1,\dots,d-1\}$ satisfies
\begin{equation}\label{eq:Reps-lower-cond}
\frac{(r+1)^{1-2\alpha}-(d+1)^{1-2\alpha}}{2\alpha-1}
\ge
\eps\,\Harm_{d,2\alpha},
\end{equation}
then
\[
R_\eps(W)>r.
\]
\end{theorem}

\begin{proof}
Since $\mu_W=\nu_\alpha$, one has
\[
\tau_\alpha(r)
=
\frac{\sum_{i=r+1}^d C^2 i^{-2\alpha}}{\sum_{i=1}^d C^2 i^{-2\alpha}}
=
\frac{\Harm_{d,2\alpha}-\Harm_{r,2\alpha}}{\Harm_{d,2\alpha}}.
\]
Thus $R_\eps(W)$ is exactly the smallest $r$ with $\tau_\alpha(r)\le \eps$.
Applying Lemma~\ref{lem:integral-test} with $s=2\alpha>1$ gives
\[
\Harm_{d,2\alpha}-\Harm_{r,2\alpha}
=
\sum_{i=r+1}^d i^{-2\alpha}
\le
\int_r^\infty x^{-2\alpha}\,dx
=
\frac{r^{1-2\alpha}}{2\alpha-1},
\]
which proves \eqref{eq:tail-condition}.
If the right-hand side is at most $\eps\Harm_{d,2\alpha}$, then $\tau_\alpha(r)\le \eps$, hence $R_\eps(W)\le r$, giving \eqref{eq:Reps-upper}.
The lower bound follows from the lower integral estimate
\[
\sum_{i=r+1}^d i^{-2\alpha}
\ge
\int_{r+1}^{d+1}x^{-2\alpha}\,dx
=
\frac{(r+1)^{1-2\alpha}-(d+1)^{1-2\alpha}}{2\alpha-1},
\]
which implies $\tau_\alpha(r)\ge \eps$ under \eqref{eq:Reps-lower-cond}.
\end{proof}

\begin{remark}[Large-width scaling]
For fixed $\alpha>\frac12$ and large $d$, $\Harm_{d,2\alpha}\to\zeta(2\alpha)$.
The upper bound in Theorem~\ref{thm:effective-rank} therefore gives the explicit leading scale
\[
R_\eps(W)
\le
\left\lceil
\left(\frac{1+o(1)}{(2\alpha-1)\eps\,\zeta(2\alpha)}\right)^{\frac{1}{2\alpha-1}}
\right\rceil
\qquad(d\to\infty).
\]
The theorem above is its exact finite-width form on the Gibbs--Cartan orbit.  For finite $d$, the quantile $R_\eps(\alpha)$ is defined for every $\alpha>0$; however, when $\alpha\le1/2$ the dimension-independent zeta-limit compression scaling is not available, and empirical summaries should report how often fitted exponents lie above this threshold.
\end{remark}

\begin{proposition}[Monotonicity of truncation rank along the Cartan orbit]
\label{prop:rank-monotone-alpha}
Fix $0<\eps<1$ and finite width $d$.
For the power-law tail measures $\nu_\alpha$ from Definition~\ref{def:tail-measure}, the cumulative mass
\[
F_r(\alpha):=\nu_\alpha(\{1,\dots,r\})
=
\sum_{i=1}^r \frac{i^{-2\alpha}}{\Harm_{d,2\alpha}}
\]
is nondecreasing in $\alpha$ for every $r=1,\dots,d$.
Equivalently, the tail mass $\tau_\alpha(r)=1-F_r(\alpha)$ is nonincreasing in $\alpha$.
Consequently, the truncation rank $R_\eps(\alpha):=\min\{r:\tau_\alpha(r)\le \eps\}$ is nonincreasing as a function of $\alpha$.
\end{proposition}

\begin{proof}
Fix $r$ and write $A=\{1,\dots,r\}$.
Using the Gibbs form $p_i^{(\alpha)}\propto e^{-2\alpha\log i}$,
\[
\frac{d}{d\alpha}p_i^{(\alpha)}
=
-2\bigl(\log i-U_d(\alpha)\bigr)p_i^{(\alpha)}.
\]
Therefore
\[
F_r'(\alpha)
=
-2\sum_{i\in A}(\log i-U_d(\alpha))p_i^{(\alpha)}
=
2F_r(\alpha)\left(U_d(\alpha)-\E_{p^{(\alpha)}}[\log i\mid i\in A]\right),
\]
with the convention that the derivative is zero when $F_r(\alpha)=0$, which never occurs here.
The conditional mean over $A$ is at most the unconditional mean $U_d(\alpha)$, because every value of $\log i$ on $A^c$ is at least every value on $A$.
Hence $F_r'(\alpha)\ge0$.
Thus $\tau_\alpha(r)=1-F_r(\alpha)$ is nonincreasing in $\alpha$.
If $\alpha_2\ge\alpha_1$ and $r=R_\eps(\alpha_1)$, then $\tau_{\alpha_2}(r)\le\tau_{\alpha_1}(r)\le\eps$, so $R_\eps(\alpha_2)\le r=R_\eps(\alpha_1)$.
\end{proof}

\begin{lemma}[Uniform Lipschitz bound for spectral tail masses]
\label{lem:tail-lipschitz-alpha}
For every $r\in\{0,1,\dots,d\}$ and every $\alpha,\beta>0$,
\begin{equation}\label{eq:tail-lipschitz-alpha}
|\tau_\alpha(r)-\tau_\beta(r)|\le 2(\log d)|\alpha-\beta|.
\end{equation}
\end{lemma}

\begin{proof}
The endpoints require no estimate: if $r=0$, then $\tau_\alpha(0)=1$ for all $\alpha$, and if $r=d$, then $\tau_\alpha(d)=0$ for all $\alpha$.  In both cases the left-hand side of \eqref{eq:tail-lipschitz-alpha} is zero.

Assume now $1\le r\le d-1$ and set
\[
F_r(\alpha):=1-\tau_\alpha(r)=\sum_{i=1}^{r}p_i^{(\alpha)}.
\]
The derivative calculation in Proposition~\ref{prop:rank-monotone-alpha} gives
\[
F_r'(\alpha)=2F_r(\alpha)
\left(U_d(\alpha)-\E_{p^{(\alpha)}}[\log i\mid i\le r]\right).
\]
The factor $F_r(\alpha)$ lies in $[0,1]$.  The random variable $\log i$ always lies in $[0,\log d]$, so both the unconditional mean $U_d(\alpha)$ and the conditional mean $\E_{p^{(\alpha)}}[\log i\mid i\le r]$ lie in that same interval.  Therefore
\[
\left|U_d(\alpha)-\E_{p^{(\alpha)}}[\log i\mid i\le r]\right|\le \log d.
\]
Combining the two bounds gives
\[
|F_r'(\alpha)|\le 2\log d.
\]
Since $\tau_\alpha(r)=1-F_r(\alpha)$, we also have
\[
\left|\frac{d}{d\alpha}\tau_\alpha(r)\right|=|F_r'(\alpha)|\le2\log d.
\]
For arbitrary $\alpha,\beta>0$, the mean value theorem applied to the continuously differentiable function $\gamma\mapsto\tau_\gamma(r)$ gives
\[
|\tau_\alpha(r)-\tau_\beta(r)|\le \sup_{\gamma\text{ between }\alpha\text{ and }\beta}\left|\frac{d}{d\gamma}\tau_\gamma(r)\right|\,|\alpha-\beta|
\le2(\log d)|\alpha-\beta|,
\]
which proves \eqref{eq:tail-lipschitz-alpha}.
\end{proof}

\begin{definition}[Rank-separation margin]
\label{def:rank-separation-margin}
Fix $0<\eps<1$ and $\alpha>0$.
Let $r_\eps(\alpha):=R_\eps(\alpha)$ denote the truncation rank of the power-law energy measure $\nu_\alpha$.
Define the rank-separation margin
\begin{equation}\label{eq:rank-separation-margin}
\mathfrak m_\eps(\alpha)
:=
\min\bigl\{\eps-\tau_\alpha(r_\eps(\alpha)),\ \tau_\alpha(r_\eps(\alpha)-1)-\eps\bigr\}.
\end{equation}
The margin is positive precisely when the threshold $\eps$ does not coincide with the tail mass at either side of the selected rank.
\end{definition}

\begin{proposition}[Stability of the effective-rank window]
\label{prop:rank-window-stability}
Fix $0<\eps<1$ and $\alpha,\beta>0$.
Let $r:=R_\eps(\alpha)$.
If
\begin{equation}\label{eq:rank-window-stability-condition}
2(\log d)|\alpha-\beta|<\mathfrak m_\eps(\alpha),
\end{equation}
then
\[
R_\eps(\beta)=R_\eps(\alpha)=r.
\]
\end{proposition}

\begin{proof}
Let $r=R_\eps(\alpha)$.  By definition of $R_\eps$, the rank $r$ is the first rank whose tail is at most $\eps$:
\[
\tau_\alpha(r)\le\eps,
\qquad
\tau_\alpha(r-1)>\eps
\]
with the second inequality interpreted for $r>1$.  The positive margin $\mathfrak m_\eps(\alpha)$ strengthens these to
\[
\tau_\alpha(r)\le \eps-\mathfrak m_\eps(\alpha),
\qquad
\tau_\alpha(r-1)\ge \eps+\mathfrak m_\eps(\alpha).
\]
The second display is exactly Definition~\ref{def:rank-separation-margin}.

By Lemma~\ref{lem:tail-lipschitz-alpha}, for every rank index $q$,
\[
|\tau_\beta(q)-\tau_\alpha(q)|\le 2(\log d)|\alpha-\beta|.
\]
Using \eqref{eq:rank-window-stability-condition}, we obtain
\[
\tau_\beta(r)
\le \tau_\alpha(r)+2(\log d)|\alpha-\beta|
< (\eps-\mathfrak m_\eps(\alpha))+\mathfrak m_\eps(\alpha)=\eps,
\]
and similarly
\[
\tau_\beta(r-1)
\ge \tau_\alpha(r-1)-2(\log d)|\alpha-\beta|
> (\eps+\mathfrak m_\eps(\alpha))-\mathfrak m_\eps(\alpha)=\eps.
\]
The first inequality says that rank $r$ captures at least $1-\eps$ of the spectral energy for parameter $\beta$.  The second says that rank $r-1$ fails to do so.  Since $R_\eps(\beta)$ is the minimal rank satisfying the tail constraint, both conditions together imply $R_\eps(\beta)=r$.
\end{proof}

\begin{corollary}[Cartan shortness selects a stable dominant-mode bundle]
\label{cor:cartan-to-rank-window}
Assume the hypotheses of Theorem~\ref{thm:cartan-rigidity} on an interval $I$.  For an interface $k$, define the theorem-predicted coordinate displacement bound
\[
B_k^{\alpha}:=
\frac{2b_k+|r_{k+1}-r_k|}{m_d(I)}.
\]
Let $r_k^{\eps}:=R_\eps(\alpha_k)$.  If
\begin{equation}\label{eq:cartan-rank-stability-condition}
2(\log d)B_k^{\alpha}<\mathfrak m_\eps(\alpha_k),
\end{equation}
then the same effective-rank window is selected on both sides of the interface:
\[
R_\eps(\alpha_{k+1})=R_\eps(\alpha_k)=r_k^{\eps}.
\]
In the exact zero-slack uniform-budget case, it is sufficient that
\[
\frac{8(\log d)(\log M)}{L\,m_d(I)}<\mathfrak m_\eps(\alpha_k).
\]
\end{corollary}

\begin{proof}
Theorem~\ref{thm:cartan-rigidity}(C2) gives
\[
|\alpha_{k+1}-\alpha_k|\le B_k^{\alpha}.
\]
If \eqref{eq:cartan-rank-stability-condition} holds, then
\[
2(\log d)|\alpha_{k+1}-\alpha_k|
\le 2(\log d)B_k^{\alpha}
<\mathfrak m_\eps(\alpha_k).
\]
This is exactly the hypothesis of Proposition~\ref{prop:rank-window-stability} with $\alpha=\alpha_k$ and $\beta=\alpha_{k+1}$.  Therefore $R_\eps(\alpha_{k+1})=R_\eps(\alpha_k)$.  In the exact zero-slack uniform-budget case, $r_{k+1}-r_k=0$, $\eta_k^{\mathrm{nb}}=0$, and $\log\lambda_k\le2\log M/L$, so $B_k^{\alpha}\le4\log M/(L m_d(I))$, giving the displayed sufficient condition.
\end{proof}

\begin{definition}[Actual spectral-tail chart residual]
\label{def:actual-tail-residual}
For a layer $W_k$ charted by $\alpha_k$, define
\begin{equation}\label{eq:actual-tail-residual}
\Delta_k^{\tau}:=\sup_{0\le r\le d}
\left|
\mu_{W_k}(\{r+1,\dots,d\})-\tau_{\alpha_k}(r)
\right|.
\end{equation}
This measures the uniform discrepancy between the actual spectral-energy tail and the fitted power-law tail.
\end{definition}

\begin{lemma}[Bures/Hellinger chart error controls tail residuals]
\label{lem:bures-controls-tail-residual}
Let $P=W^\top W\in\SPD_d^{(1)}$ and let $\mu_W$ be its spectral energy measure.  If $P$ is charted by $\alpha$, then
\[
\Delta^\tau
:=\sup_{0\le r\le d}|\mu_W(\{r+1,\dots,d\})-\tau_\alpha(r)|
\le
\|\mu_W-\nu_\alpha\|_{\mathrm{TV}}
\le
\bar e^{\BW},
\]
where $\bar e^{\BW}=d^{-1/2}d_{\BW}(\Cart(P),G_d(\alpha))=\|\sqrt{\mu_W}-\sqrt{\nu_\alpha}\|_2$.
\end{lemma}

\begin{proof}
The first inequality is the definition of total variation as a supremum over measurable sets.  For the second, write $h=\|\sqrt{\mu_W}-\sqrt{\nu_\alpha}\|_2$.  Then
\[
\|\mu_W-\nu_\alpha\|_1
\le
\sum_i |\sqrt{\mu_i}-\sqrt{\nu_i}|(\sqrt{\mu_i}+\sqrt{\nu_i})
\le
h\,\|\sqrt{\mu_W}+\sqrt{\nu_\alpha}\|_2
\le 2h.
\]
Thus $\|\mu_W-\nu_\alpha\|_{\mathrm{TV}}=\frac12\|\mu_W-\nu_\alpha\|_1\le h$.  The equality $h=\bar e^{\BW}$ follows from the diagonal trace-$d$ Bures formula.
\end{proof}

\begin{proposition}[Actual effective-rank stability under tail residuals]
\label{prop:actual-rank-window-stability}
Assume the hypotheses of Theorem~\ref{thm:cartan-rigidity} and define $B_k^{\alpha}$ as in Corollary~\ref{cor:cartan-to-rank-window}.  Let $r_k^{\eps}:=R_\eps(\alpha_k)$.  If
\begin{equation}\label{eq:actual-rank-stability-condition}
2(\log d)B_k^{\alpha}+\Delta_k^{\tau}+\Delta_{k+1}^{\tau}
<\mathfrak m_\eps(\alpha_k),
\end{equation}
then the actual layerwise effective-rank windows agree with the fitted window across the interface:
\begin{equation}\label{eq:actual-rank-stability-conclusion}
R_\eps(W_k)=R_\eps(W_{k+1})=R_\eps(\alpha_k)=r_k^{\eps}.
\end{equation}
\end{proposition}

\begin{proof}
Let $r=r_k^{\eps}$.  By Definition~\ref{def:rank-separation-margin},
\[
\tau_{\alpha_k}(r)\le \eps-\mathfrak m_\eps(\alpha_k),
\qquad
\tau_{\alpha_k}(r-1)\ge \eps+\mathfrak m_\eps(\alpha_k).
\]
Since \eqref{eq:actual-rank-stability-condition} implies $\Delta_k^{\tau}<\mathfrak m_\eps(\alpha_k)$, the actual tail of $W_k$ is at most $\eps$ at rank $r$ and greater than $\eps$ at rank $r-1$, so $R_\eps(W_k)=r$.  Theorem~\ref{thm:cartan-rigidity}(C2) gives $|\alpha_{k+1}-\alpha_k|\le B_k^{\alpha}$.  Lemma~\ref{lem:tail-lipschitz-alpha} then gives
\[
|\tau_{\alpha_{k+1}}(q)-\tau_{\alpha_k}(q)|
\le 2(\log d)B_k^{\alpha}
\]
for every $q$.  Combining this with \eqref{eq:actual-rank-stability-condition} and the tail residual $\Delta_{k+1}^{\tau}$ shows that the actual tail of $W_{k+1}$ is also at most $\eps$ at rank $r$ and greater than $\eps$ at rank $r-1$.  Hence $R_\eps(W_{k+1})=r$.
\end{proof}

\begin{proposition}[Single-layer sandwich from tail residuals]
\label{prop:actual-rank-sandwich}
If a layer $W$ charted by $\alpha$ has tail residual $\Delta^\tau$ and $0<\Delta^\tau<\eps<1-\Delta^\tau$, then
\begin{equation}\label{eq:actual-rank-sandwich}
R_{\eps+\Delta^\tau}(\alpha)
\le
R_\eps(W)
\le
R_{\eps-\Delta^\tau}(\alpha).
\end{equation}
\end{proposition}

\begin{proof}
If $r=R_{\eps-\Delta^\tau}(\alpha)$, then $\tau_\alpha(r)\le\eps-\Delta^\tau$, so the actual tail of $W$ at $r$ is at most $\eps$; hence $R_\eps(W)\le r$.  If $r=R_{\eps+\Delta^\tau}(\alpha)$, then every $q<r$ has $\tau_\alpha(q)>\eps+\Delta^\tau$, so the actual tail of $W$ at $q$ is greater than $\eps$; hence $R_\eps(W)\ge r$.
\end{proof}


\subsection{Effective-rank transitions}

\begin{definition}[Effective-rank plateaux]
\label{def:rank-plateaux}
Fix $0<\eps<1$.  For $r=1,\dots,d$, define
\begin{equation}\label{eq:rank-plateau}
\mathcal I_r^{\eps}:=\bigl\{\alpha>0:\ \tau_\alpha(r)\le\eps<\tau_\alpha(r-1)\bigr\},
\end{equation}
where $\tau_\alpha(0)=1$ and $\tau_\alpha(d)=0$.
\end{definition}

\begin{theorem}[Localization of effective-rank transitions]
\label{thm:rank-transition-localization}
Fix $0<\eps<1$.
\begin{enumerate}[label=(R\arabic*),leftmargin=2.1em]
\item For every $r=1,\dots,d$ and every $\alpha>0$,
\[
R_\eps(\alpha)=r
\quad\Longleftrightarrow\quad
\alpha\in\mathcal I_r^{\eps}.
\]
Each $\mathcal I_r^{\eps}$ is an interval, possibly empty.
\item Let $(\alpha_k)_{k=0}^{L-1}$ be any finite sequence, and let
\[
\mathcal T_\eps:=\{k:R_\eps(\alpha_{k+1})\ne R_\eps(\alpha_k)\}.
\]
If $k\in\mathcal T_\eps$, then
\begin{equation}\label{eq:transition-local-lower}
2(\log d)|\alpha_{k+1}-\alpha_k|
\ge
\mathfrak m_\eps(\alpha_k).
\end{equation}
Consequently, if there exists $\mu>0$ such that $\mathfrak m_\eps(\alpha_k)\ge\mu$ for every $k\in\mathcal T_\eps$, then
\begin{equation}\label{eq:transition-count-TV}
|\mathcal T_\eps|
\le
\frac{2\log d}{\mu}
\sum_{k=0}^{L-2}|\alpha_{k+1}-\alpha_k|.
\end{equation}
\item Under the hypotheses of Theorem~\ref{thm:cartan-rigidity} on the interval $I$, if the same lower margin bound $\mu$ holds on $\mathcal T_\eps$, then
\begin{equation}\label{eq:transition-count-budget}
|\mathcal T_\eps|
\le
\frac{2\log d}{\mu\,m_d(I)}
\left(
2\sum_{k=0}^{L-2}b_k
+\TV(r)
\right).
\end{equation}
If only $|r_k|\le\bar\delta$ is known, one may replace $\TV(r)$ by $2(L-1)\bar\delta$.
\end{enumerate}
\end{theorem}

\begin{proof}
For (R1), recall that $R_\eps(\alpha)$ is the smallest rank $r$ satisfying $\tau_\alpha(r)\le\eps$.  Therefore $R_\eps(\alpha)=r$ exactly when $\tau_\alpha(r)\le\eps$ and $\tau_\alpha(r-1)>\eps$, which is exactly \eqref{eq:rank-plateau}.  Proposition~\ref{prop:rank-monotone-alpha} shows that each tail $\tau_\alpha(s)$ is nonincreasing in $\alpha$.  Hence $\{\alpha:\tau_\alpha(r)\le\eps\}$ is an upper interval and $\{\alpha:\tau_\alpha(r-1)>\eps\}$ is a lower interval.  Their intersection is an interval, possibly empty.

For (R2), fix $k\in\mathcal T_\eps$.  If \eqref{eq:transition-local-lower} failed, then
\[
2(\log d)|\alpha_{k+1}-\alpha_k|<\mathfrak m_\eps(\alpha_k).
\]
Proposition~\ref{prop:rank-window-stability}, applied with $\alpha=\alpha_k$ and $\beta=\alpha_{k+1}$, would imply $R_\eps(\alpha_{k+1})=R_\eps(\alpha_k)$, contradicting the definition of $\mathcal T_\eps$.  Thus \eqref{eq:transition-local-lower} holds.  If $\mathfrak m_\eps(\alpha_k)\ge\mu$ on $\mathcal T_\eps$, then each transition satisfies
\[
|\alpha_{k+1}-\alpha_k|\ge \frac{\mu}{2\log d}.
\]
Summing this over $k\in\mathcal T_\eps$ gives
\[
|\mathcal T_\eps|\frac{\mu}{2\log d}
\le
\sum_{k\in\mathcal T_\eps}|\alpha_{k+1}-\alpha_k|
\le
\sum_{k=0}^{L-2}|\alpha_{k+1}-\alpha_k|,
\]
which proves \eqref{eq:transition-count-TV}.

For (R3), substitute the total-variation estimate from Theorem~\ref{thm:cartan-rigidity}(C3) into \eqref{eq:transition-count-TV}.
\end{proof}

\begin{remark}
This theorem separates two effects.  The effective rank can change only when the exponent coordinate crosses a spectral-tail threshold, and the number of such crossings is bounded by the same Cartan total-variation budget that controls the exponent path.
\end{remark}

\section{Empirical measurements for Cartan-coordinate rigidity}
\label{sec:empirical}

The spectral theorems predict short trajectories for the fitted Cartan power-law coordinate when local interface margins and orbit-chart residual variation are controlled.  In trained models, the finite-dimensional quantity measured directly in this version is the fitted power-law exponent of static layerwise weight operators.  This section reports that coordinate measurement across residual CNNs, LLMs, and vision/diffusion backbones.  These plots are not by themselves a verification of Theorem~\ref{thm:cartan-rigidity}, because the theorem applies to a specified operator chain and its adjacent products.  A direct Jacobian verification would compute or approximate block Jacobians $J_k(x)=I+DF_k(x)$ on data and then report the same margin quantities.  The complete finite-dimensional margin test also requires the interface margins, slacks, and chart residuals listed in Table~\ref{tab:cartan-measurement-protocol}.

\subsection{Measurement protocol}
\label{sec:exp1-protocol}

For each selected layerwise matrix $W_k$, we compute singular values and fit a power-law profile
\[
\sigma_i(W_k)\approx C_k i^{-\alpha_k}
\]
on a prescribed rank window.  Exact SVD is used when dimensions permit; for larger matrices the same singular-value quantities can be estimated with standard truncated or randomized SVD methods \cite{halko2011finding}.  The fitted exponent $\widehat\alpha_k$ is the empirical coordinate corresponding to the Cartan power-law orbit $G_d(\alpha)$.  We record the depthwise sequence
\[
(\widehat\alpha_0,\widehat\alpha_1,\dots,\widehat\alpha_{L-1})
\]
and compare it with the theorem-predicted short-trajectory behavior.  The most direct empirical summary is the total variation
\[
\mathrm{TV}(\widehat\alpha):=\sum_{k=0}^{L-2}|\widehat\alpha_{k+1}-\widehat\alpha_k|.
\]
Theorem~\ref{thm:cartan-rigidity} gives the deterministic bound for the corresponding exact or charted coordinates once the local margins, non-backtracking slacks, and chart residual increments are measured.  The plots below measure the visible coordinate trajectory; a complete finite-dimensional margin test additionally reports $\widehat\Lambda_k$, $\widehat\eta_k^{\mathrm{nb}}$, the combined margin $\widehat b_k$, the signed radial residuals $\widehat r_k$, normalized Bures/Hellinger errors $\widehat{\bar e}_k^{\BW}$, and the local margin ratio.

\begin{table}[H]
\centering
\small
\renewcommand{\arraystretch}{1.22}
\setlength{\tabcolsep}{3.0pt}
\begin{tabularx}{\textwidth}{L{0.24\textwidth}L{0.34\textwidth}Y}
\toprule
\textbf{Measured quantity} & \textbf{How it is computed} & \textbf{Meaning in the theory}\\
\midrule
Fitted Cartan coordinate & Fit $\sigma_i(W_k)\approx C_ki^{-\alpha_k}$ and record $\widehat\alpha_k$. & Empirical coordinate on the Cartan power-law orbit $G_d(\alpha)$.\\
Coordinate increment & $|\widehat\alpha_{k+1}-\widehat\alpha_k|$ between adjacent layers. & Measured version of local coordinate displacement in Theorem~\ref{thm:cartan-rigidity}.\\
Total variation & $\sum_k|\widehat\alpha_{k+1}-\widehat\alpha_k|$. & Measured version of the short fitted coordinate path predicted by the theorem.\\
Interface radial amplitude & $\widehat\Lambda_k=\|W_{k+1}W_k\|_2/\sqrt{\|W_{k+1}\|_2\|W_k\|_2}$ when adjacent products are dimension-compatible. & Geometric-mean radial amplitude in Definition~\ref{def:lambda}.\\
Non-backtracking slack & Compute $\widehat\eta_k^{\mathrm{nb}}$ from Definition~\ref{def:nonbacktracking-slack} using measured operator norms. & Measured defect from the zero-slack non-backtracking case.\\
Combined local margin & $\widehat b_k=\log\widehat\lambda_k+\widehat\eta_k^{\mathrm{nb}}$. & Scalar transport margin in Theorem~\ref{thm:cartan-rigidity}.\\
Radial residual increment & $|\widehat r_{k+1}-\widehat r_k|$ with $\widehat r_k=\rho_d(W_k^\top W_k)-g_d(\widehat\alpha_k)$. & Signed chart-residual variation in the main theorem.\\
Local margin ratio & $m_d(I)|\widehat\alpha_{k+1}-\widehat\alpha_k|/(2\widehat b_k+|\widehat r_{k+1}-\widehat r_k|)$. & Direct check of the deterministic local inequality when the denominator is nonzero.\\
Full chart error & $\widehat e_k^{\BW}=d_{\BW}(\Cart(W_k^\top W_k),G_d(\widehat\alpha_k))$ or normalized $\widehat{\bar e}_k^{\BW}=\widehat e_k^{\BW}/\sqrt d$. & Empirical gap between the actual Cartan spectrum and the fitted orbit path.\\
Effective-rank window & $R_\eps(W_k)$ computed from the spectral energy measure. & Dominant spectral window whose transitions are controlled by the Cartan total-variation bound.\\
\bottomrule
\end{tabularx}
\caption{Finite-dimensional spectral measurements associated with the slack-aware Cartan-coordinate theorem.  The figures display the fitted coordinate trajectory; complete verification also requires checking the interface amplitude, combined margin, non-backtracking slack, and chart-error margins.}
\label{tab:cartan-measurement-protocol}
\end{table}

\subsection{Empirical measurement I: short Cartan-coordinate trajectories}
\label{sec:exp1}

The fitted exponent $\widehat\alpha_k$ is a low-dimensional summary of how the singular-value energy of a layer is distributed across rank.  Smoothness of the sequence $\widehat\alpha_k$ indicates that representation complexity changes gradually across depth rather than undergoing abrupt spectral shocks between adjacent layers.  This is a static-weight diagnostic for the fitted Cartan-coordinate path in Theorem~\ref{thm:cartan-rigidity}.  It is not a direct theorem verification unless the plotted matrices are the same operators used in the chain and the adjacent products are dimension-compatible.  The theorem's actual margin variables are the combined interface margin, slack, and chart residuals; the coordinate plots should therefore be read as the first measurement in that margin pipeline.

\begin{figure}[H]
\centering
\begin{subfigure}{0.32\textwidth}
\centering
\includegraphics[width=\textwidth]{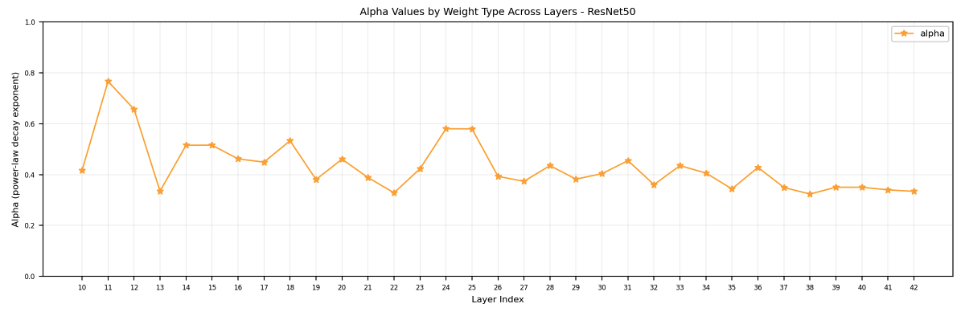}
\caption{ResNet50}
\end{subfigure}\hfill
\begin{subfigure}{0.32\textwidth}
\centering
\includegraphics[width=\textwidth]{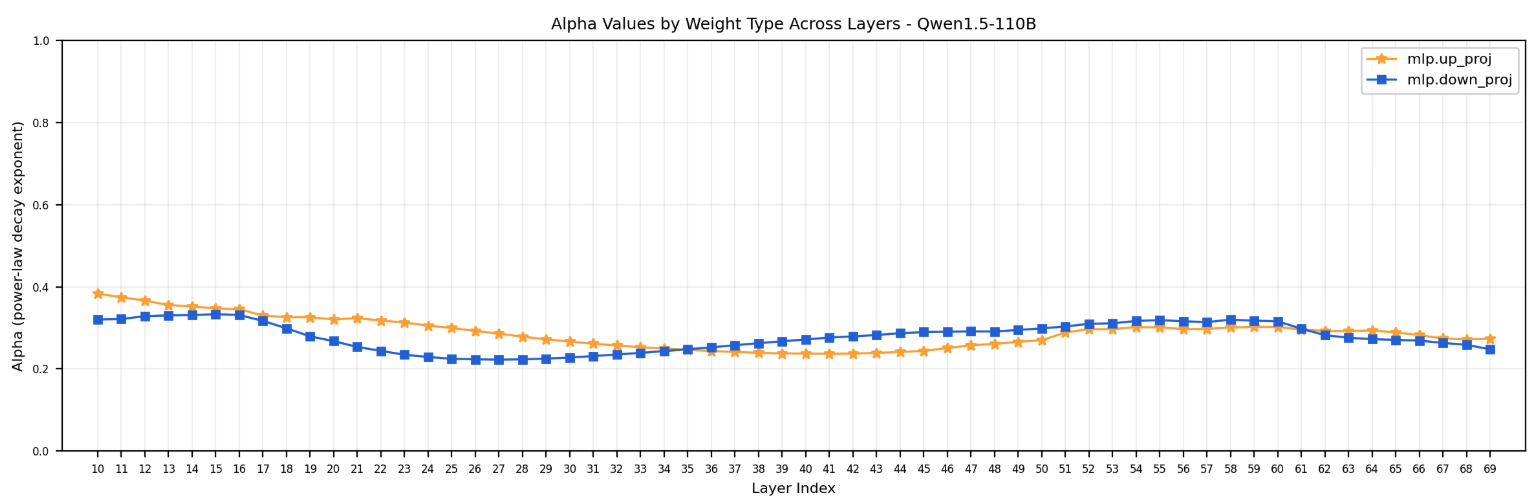}
\caption{Qwen1.5-110B}
\end{subfigure}\hfill
\begin{subfigure}{0.32\textwidth}
\centering
\includegraphics[width=\textwidth]{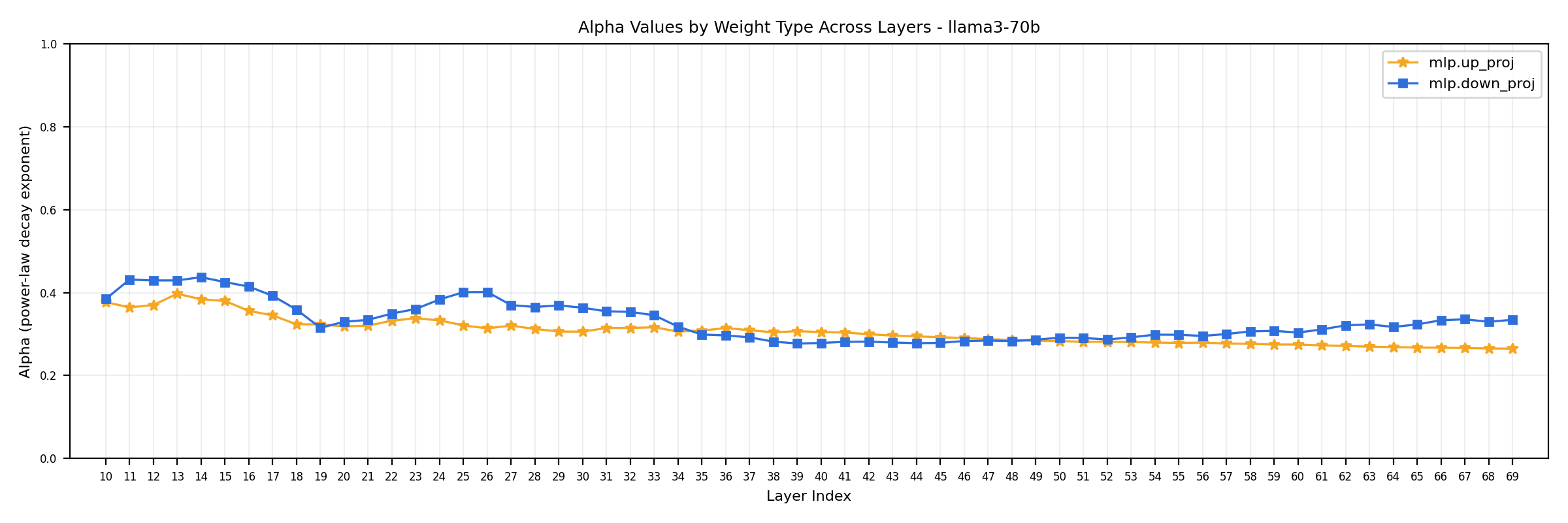}
\caption{Llama3-70B}
\end{subfigure}

\par\medskip

\begin{subfigure}{0.32\textwidth}
\centering
\includegraphics[width=\textwidth]{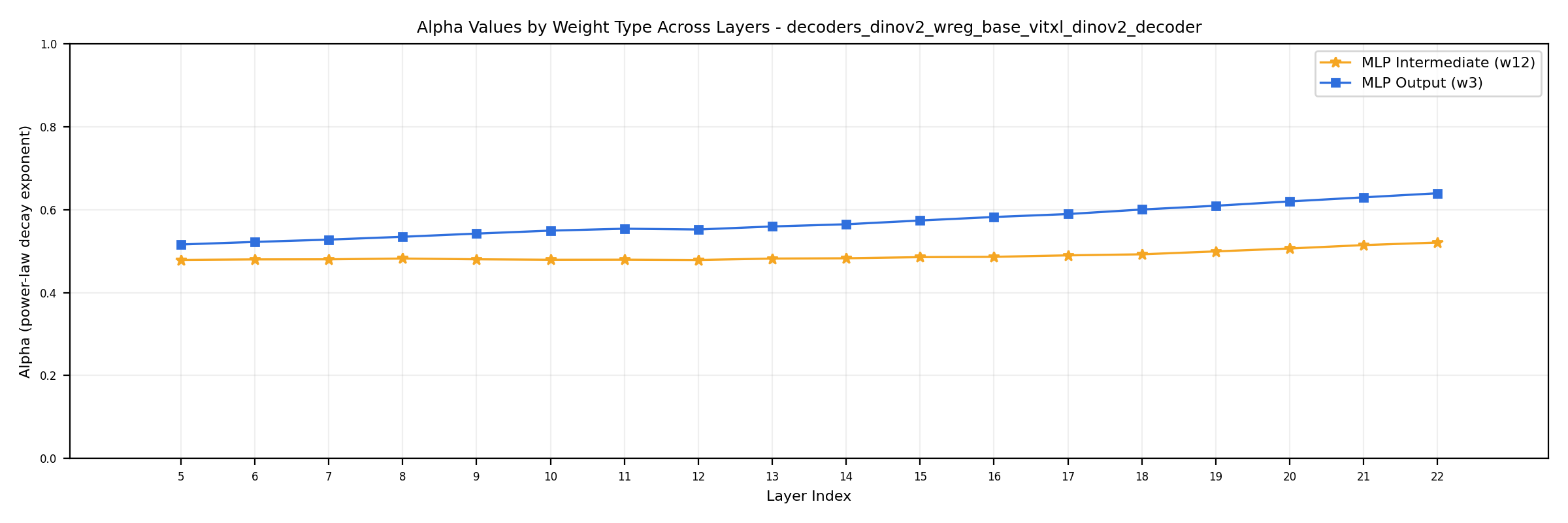}
\caption{DINOv2 decoder}
\end{subfigure}\hspace{0.02\textwidth}
\begin{subfigure}{0.32\textwidth}
\centering
\includegraphics[width=\textwidth]{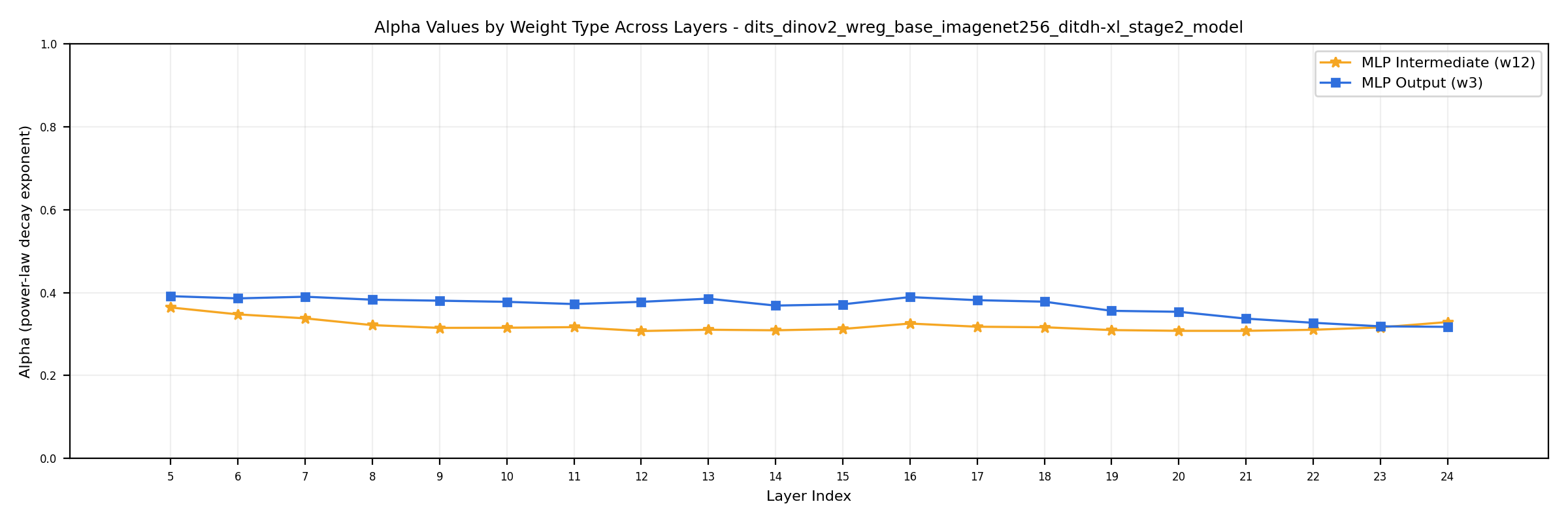}
\caption{DiT-DH-XL on ImageNet256}
\end{subfigure}

\caption{\textbf{Spectral-coordinate measurements across model families.}
Each panel plots the fitted power-law exponent $\widehat\alpha_k$ as a function of layer index.  Residual CNNs, large language models, and vision/diffusion backbones all display structured depthwise trajectories rather than layerwise random oscillation.  These plots measure the fitted coordinate appearing in Theorem~\ref{thm:cartan-rigidity}; a stronger finite-dimensional margin check additionally reports interface amplitudes, combined local margins, non-backtracking slacks, signed radial residuals, and full chart errors.}
\label{fig:alpha-all}
\end{figure}

\paragraph{Meaning of the measurement.}
The plots display the empirical fitted coordinate to which Theorem~\ref{thm:cartan-rigidity} applies.  Under the stated budget, slack, and chart-residual hypotheses, the theorem bounds the variation of this coordinate.  The observed smooth trajectories are consistent with that finite-dimensional behavior; full verification requires measuring $\widehat\Lambda_k$, $\widehat\eta_k^{\mathrm{nb}}$, the combined margin $\widehat b_k$, $|\widehat r_{k+1}-\widehat r_k|$, and a full-spectrum chart error such as $\widehat e_k^{\BW}$ or its normalized version $\widehat{\bar e}_k^{\BW}$.

\section{Conclusion}
\label{sec:conclusion}
This paper establishes slack-aware quotient-radial and spectral-measure rigidity estimates for GSA.  Residual Jacobian chains are modeled as finite cocycles in $\GL(d)$, their singular spectra are represented on the quotient state space $\SPD(d)$, and normalized power-law spectra form a one-dimensional Cartan orbit.  The main result is a deterministic margin theorem: local geometric-mean radial amplitudes, measurable non-backtracking slacks, and signed orbit-chart residual variation imply short displacement of the fitted exponent coordinate.  Full Cartan-spectral shortness is obtained only after adding explicit full-spectrum chart-residual variation.  The spectral energy measure gives a parallel finite-dimensional theory of effective-rank windows.

These results provide the spectral inputs for the companion article, which studies angular transport and static channel incidence.  Together, the two papers form a single GSA theory: this article controls the fitted spectral coordinate, the stability of dominant rank windows, and the chart errors needed to compare fitted and actual spectra; the companion article studies how the dominant singular directions inside those windows are transported and how their static channel structure can be extracted.

\subsection{Relation to strict balancedness in deep linear networks}
\label{sec:strict-balanced}

Strict balancedness in deep linear networks is the algebraic condition studied in the deep-linear optimization literature~\cite{arora2018optimization,arora2019implicitmf,menonyu2025entropy}:
\[
W_k^\top W_k=W_{k+1}W_{k+1}^\top .
\]
It fixes the singular spectrum exactly and, after a gauge choice, enforces rigid transport.  The spectral GSA results in this paper give a metric relaxation of this algebraic equality: exact spectral equality is replaced by quantitative bounds on fitted Cartan-coordinate drift under explicit interface budgets, slack terms, and chart residuals.  The angular and static-channel relaxations are developed separately in the companion article, titled by Geometric Rigidity of Residual Jacobian Chains II.

\begin{proposition}[Strict balancedness is a zero-thickness special case]
\label{prop:strict-balanced-special-case}
Assume strict balancedness and exact power-law spectra
\[
\sigma_i(W_k)=C_k i^{-\alpha_k}.
\]
Then $C_{k+1}=C_k$ and $\alpha_{k+1}=\alpha_k$ for every adjacent pair.
Thus strict balancedness is a zero-drift, zero-thickness corner of the spectral GSA regime studied here.
\end{proposition}

\begin{proof}
Strict balancedness implies equality of the positive semidefinite matrices $W_k^\top W_k$ and $W_{k+1}W_{k+1}^\top$.  Their eigenvalues are the squared singular values of $W_k$ and $W_{k+1}$, respectively.  Hence the singular-value multisets coincide.  Under the ordered exact power-law representation, equality at $i=1$ gives $C_{k+1}=C_k$.  With this equality of constants, equality at $i=2$ gives $2^{-\alpha_{k+1}}=2^{-\alpha_k}$, hence $\alpha_{k+1}=\alpha_k$.
\end{proof}

\printbibliography
\newpage
\appendix

\section{Large-width asymptotics and the zeta constants}
\label{sec:width_zeta_discussion}

All main results in the paper are finite-width and stated in terms of generalized harmonic numbers.
This appendix records the standard large-width bridge to zeta-function notation.

\begin{proposition}[Tail bounds for generalized harmonic numbers]
\label{prop:harm-tail}
Let $s>1$ and $d\ge 1$. Then
\[
0\le \zeta(s)-\Harm_{d,s}\le \frac{d^{1-s}}{s-1}.
\]
\end{proposition}

\begin{proof}
For $s>1$, the function $f(x)=x^{-s}$ is positive and decreasing on $[1,\infty)$.  For each integer $n\ge d+1$, monotonicity gives
\[
\int_n^{n+1} f(x)\,dx\le f(n)\le \int_{n-1}^{n}f(x)\,dx.
\]
Summing the left inequality over $n=d+1,d+2,\dots$ gives
\[
\int_{d+1}^{\infty}x^{-s}\,dx\le \sum_{n=d+1}^{\infty}n^{-s},
\]
and summing the right inequality gives
\[
\sum_{n=d+1}^{\infty}n^{-s}\le \int_d^{\infty}x^{-s}\,dx.
\]
Since
\[
\int_d^{\infty}x^{-s}\,dx=\left[\frac{x^{1-s}}{1-s}\right]_{d}^{\infty}=\frac{d^{1-s}}{s-1},
\]
we obtain
\[
0\le \zeta(s)-\Harm_{d,s}=\sum_{n=d+1}^{\infty}n^{-s}\le \frac{d^{1-s}}{s-1}.
\]
\end{proof}

\begin{lemma}[Termwise differentiation of the Dirichlet series]
\label{lem:zeta-derivative}
Fix $s_0>1$.
Then $\sum_{n=1}^\infty n^{-s}$ and $\sum_{n=1}^\infty (\log n)n^{-s}$ converge uniformly on $[s_0,\infty)$, and
\[
\zeta'(s)=-\sum_{n=1}^\infty (\log n)\,n^{-s},
\qquad s>1.
\]
\end{lemma}

\begin{proof}
Fix $s_0>1$.  For every $s\ge s_0$ and every $n\ge1$,
\[
n^{-s}\le n^{-s_0}.
\]
Since $\sum_{n=1}^{\infty}n^{-s_0}<\infty$, the Weierstrass $M$-test gives uniform convergence of $\sum_n n^{-s}$ on $[s_0,\infty)$.  The derivative of the $n$-th term is
\[
\frac{d}{ds}n^{-s}=-(\log n)n^{-s}.
\]
Choose $s_1$ with $1<s_1<s_0$.  For $s\ge s_0$ and all sufficiently large $n$, one has $\log n\le n^{s_0-s_1}$, and hence
\[
(\log n)n^{-s}\le (\log n)n^{-s_0}\le n^{-s_1}.
\]
The series $\sum_n n^{-s_1}$ converges, so another $M$-test gives uniform convergence of the derivative series on $[s_0,\infty)$.  The standard termwise differentiation theorem for series of $C^1$ functions then gives
\[
\zeta'(s)=\sum_{n=1}^{\infty}\frac{d}{ds}n^{-s}=-\sum_{n=1}^{\infty}(\log n)n^{-s},\qquad s>s_0.
\]
Because $s_0>1$ was arbitrary, the identity holds for every $s>1$.
\end{proof}

\begin{proposition}[Tail bounds for log-weighted sums]
\label{prop:log-tail}
Let $s>1$ and $d\ge 3$. Then
\[
0\le
\Bigl(-\zeta'(s)\Bigr)-\sum_{n=1}^{d} (\log n)n^{-s}
\le
d^{1-s}\left(\frac{\log d}{s-1}+\frac{1}{(s-1)^2}\right).
\]
\end{proposition}

\begin{proof}
Let $g(x)=(\log x)x^{-s}$.  Then
\[
g'(x)=x^{-s-1}(1-s\log x).
\]
For $x\ge3$ and $s>1$, $1-s\log x<0$, so $g$ is positive and decreasing on $[3,\infty)$.  Applying the same integral-test argument as in Proposition~\ref{prop:harm-tail}, for $d\ge3$ we obtain
\[
0\le \sum_{n=d+1}^{\infty}(\log n)n^{-s}\le \int_d^{\infty}(\log x)x^{-s}\,dx.
\]
An integration by parts gives
\[
\int_d^{\infty}(\log x)x^{-s}\,dx
=\left[\frac{x^{1-s}\log x}{1-s}\right]_{d}^{\infty}-\int_d^{\infty}\frac{x^{1-s}}{1-s}\frac{dx}{x}
=d^{1-s}\left(\frac{\log d}{s-1}+\frac{1}{(s-1)^2}\right).
\]
By Lemma~\ref{lem:zeta-derivative},
\[
-\zeta'(s)-\sum_{n=1}^{d}(\log n)n^{-s}=\sum_{n=d+1}^{\infty}(\log n)n^{-s}.
\]
Combining these identities proves the stated bound.
\end{proof}

\begin{corollary}[Zeta-form limit of the coordinate derivative]
\label{cor:zeta-limit}
Fix $\alpha>\frac12$ and set $s:=2\alpha$.
Then as $d\to\infty$,
\[
\Harm_{d,s}\to \zeta(s),
\qquad
\sum_{i=1}^d (\log i)\,i^{-s}\to -\zeta'(s),
\]
hence
\[
g_d'(\alpha)\to \frac{-\zeta'(2\alpha)}{\zeta(2\alpha)}.
\]
Moreover, the convergence rate is $O(d^{1-2\alpha}\log d)$.
\end{corollary}

\begin{proof}
Let $s=2\alpha>1$.  Proposition~\ref{prop:harm-tail} gives $\Harm_{d,s}\to\zeta(s)$.  Lemma~\ref{lem:zeta-derivative} identifies
\[
\sum_{i=1}^{\infty}(\log i)i^{-s}=-\zeta'(s),
\]
and Proposition~\ref{prop:log-tail} gives convergence of the corresponding partial sums.  Therefore
\[
g_d'(\alpha)=\frac{\sum_{i=1}^{d}(\log i)i^{-2\alpha}}{\Harm_{d,2\alpha}}
\longrightarrow
\frac{-\zeta'(2\alpha)}{\zeta(2\alpha)}.
\]
The denominator converges to the positive finite number $\zeta(s)$, while the numerator and denominator tails are bounded by $O(d^{1-s}\log d)$ and $O(d^{1-s})$, respectively, by Propositions~\ref{prop:harm-tail} and~\ref{prop:log-tail}.  Dividing by a denominator bounded away from zero for all sufficiently large $d$ gives the stated $O(d^{1-2\alpha}\log d)$ rate.
\end{proof}

\end{document}